\documentclass[journal, 10pt]{IEEEtran}

\usepackage{marvosym}

\usepackage{ieee-con-nlw}

\begin{document}
\begin{CJK}{UTF8}{gbsn}

\title{OpenLS-DGF: An Adaptive Open-Source Dataset Generation Framework for Machine Learning Tasks \\in Logic Synthesis}

%%%%%%%%%%%%%%      arxiv
\author{
\IEEEauthorblockN{Liwei~Ni\textsuperscript{1,2,3,}$^{\star}$,}
\and
\IEEEauthorblockN{Rui~Wang\textsuperscript{4},}
\and
\IEEEauthorblockN{Miao~Liu\textsuperscript{3},}
\and
\IEEEauthorblockN{Xingyu~Meng\textsuperscript{2},}
\and
\IEEEauthorblockN{Xiaoze~Lin\textsuperscript{1,2,3},}
\and
\IEEEauthorblockN{Junfeng~Liu\textsuperscript{2},}
\and
\IEEEauthorblockN{Guojie~Luo\textsuperscript{5},}
\and
\IEEEauthorblockN{Zhufei~Chu\textsuperscript{6},}
\and
\IEEEauthorblockN{Weikang~Qian\textsuperscript{7},}
\and
\IEEEauthorblockN{Xiaoyan~Yang\textsuperscript{8},}
\and
\IEEEauthorblockN{Biwei~Xie\textsuperscript{1,2,3},}
\and
\IEEEauthorblockN{Xingquan~Li\textsuperscript{2,}$^\text{\Letter}$,}
 \and
\IEEEauthorblockN{Huawei~Li\textsuperscript{1,2,3}} \\

\IEEEauthorblockA{\textsuperscript{1}Institute of Computing Technology, Chinese Academy of Sciences, Beijing, China}  \\
\IEEEauthorblockA{\textsuperscript{2}Pengcheng Laboratory, Shenzhen, China}  \\
\IEEEauthorblockA{\textsuperscript{3}University of Chinese Academy of Sciences, Beijing, China}  \\
\IEEEauthorblockA{\textsuperscript{4}Shenzhen University, Shenzhen, China}  \\
\IEEEauthorblockA{\textsuperscript{5}Peking University, Beijing, China}  \\
\IEEEauthorblockA{\textsuperscript{6}Ningbo University, Ningbo, China} \\
\IEEEauthorblockA{\textsuperscript{7}Shanghai Jiao Tong University, Shanghai, China}  \\
\IEEEauthorblockA{\textsuperscript{8}Hangzhou Dianzi University, Hangzhou, China} \\

Emails: \{nlwmode@gmail.com$^{\star}$, lixq01@pcl.ac.cn$^{\text{\Letter}}$\} 
}

\maketitle

%%%%%%%%%%%%%%%%%%%%%%%%%%%%%%%%%%%%%%%%%%%%%%%%%%%%%%%%%%%%%%%%%%%%%%%%%%%%
%   abstract
%%%%%%%%%%%%%%%%%%%%%%%%%%%%%%%%%%%%%%%%%%%%%%%%%%%%%%%%%%%%%%%%%%%%%%%%%%%%
\begin{abstract}

This paper introduces OpenLS-DGF, an adaptive logic synthesis dataset generation framework, to enhance machine learning~(ML) applications within the logic synthesis process.
Previous dataset generation flows were tailored for specific tasks or lacked integrated machine learning capabilities.
While OpenLS-DGF supports various machine learning tasks by encapsulating the three fundamental steps of logic synthesis: Boolean representation, logic optimization, and technology mapping.
It preserves the original information in both Verilog and machine-learning-friendly GraphML formats.
The verilog files offer semi-customizable capabilities, enabling researchers to insert additional steps and incrementally refine the generated dataset.
Furthermore, OpenLS-DGF includes an adaptive circuit engine that facilitates the final dataset management and downstream tasks.
The generated OpenLS-D-v1 dataset comprises 46 combinational designs from established benchmarks, totaling over 966,000 Boolean circuits. 
OpenLS-D-v1 supports integrating new data features, making it more versatile for new challenges.
This paper demonstrates the versatility of OpenLS-D-v1 through four distinct downstream tasks: circuit classification, circuit ranking, quality of results (QoR) prediction, and probability prediction.
Each task is chosen to represent essential steps of logic synthesis, and the experimental results show the generated dataset from OpenLS-DGF achieves prominent diversity and applicability.
The source code and datasets are available at \href{https://github.com/Logic-Factory/ACE/blob/master/OpenLS-DGF/readme.md}{https://github.com/Logic-Factory/ACE/blob/master/OpenLS-DGF/readme.md}.
\end{abstract}


%%%%%%%%%%%%%%%%%%%%%%%%%%%%%%%%%%%%%%%%%%%%%%%%%%%%%%%%%%%%%%%%%%%%%%%%%%%%
%   intro
%%%%%%%%%%%%%%%%%%%%%%%%%%%%%%%%%%%%%%%%%%%%%%%%%%%%%%%%%%%%%%%%%%%%%%%%%%%%
\section{Introduction}
\label{sec:intro}

\IEEEPARstart{L}{ogic} synthesis is a key phase in the electronic design automation~(EDA) flow of digital circuits, translating high-level specifications into a gate-level netlist.
Recently, there has been a trend towards adopting ML approaches for the EDA~\cite{MLEDA} domain.
Various machine learning methodologies have been proposed, demonstrating improvements in different aspects of the logic synthesis process, including logic optimization~\cite{rai2020logicsynthesismeetsmachine, deeplearningforlogicoptimization, lsoracle, aisyn, ni2023adaptive}, technology mapping~\cite{slap, aimap, easymap}, and formal verification~\cite{GNN-RE, Gamora}.
These machine learning-based techniques have shown their promise in improving the efficiency and quality of logic synthesis steps.
In order to further develop these techniques, it is crucial to introduce more comprehensive and reliable datasets.

Previous benchmarks~\cite{ISCAS85, ISCAS89, MCNC91, IWLS93, IWLS2005, IWLS2015, opencores} provide a foundation for testing, comparison, and enhancement, significantly advancing the development of EDA tools and methodologies.
Moreover, logic synthesis datasets~\cite{openabc-d, DeepGate, Gamora} such as OpenABC-D, have been derived from these foundational benchmarks.
However, these datasets are often tailored for specific tasks, limiting their use cases for diverse applications in machine learning.
This paper underscores the need for a more versatile and adaptive dataset generation framework capable of supporting a variety of machine-learning tasks in logic synthesis.
Such a framework should ideally possess the following attributes:
\begin{itemize}
    \item Diversity: Generating a dataset covers a wide range of design types and categories, ensuring it can cater to a diverse array of use cases and applications;
    \item Versatility: Generating a dataset has the capacity to support various machine learning tasks, facilitating the sharing of the same dataset across different tasks;
    \item Adaptivity: Generating a dataset can adapt to different tasks, enabling the extraction of sub-datasets tailored to specific downstream tasks.
\end{itemize}
While the EDA flows like OpenLane~\cite{OpenLANE, openroad} are primarily aimed at facilitating the chip tape-out process, they do not inherently provide the specific needs for dataset generation.
This further emphasizes the demand for a dedicated, adaptable dataset framework within the logic synthesis domain.

To address these limitations, we introduce OpenLS-DGF, an adaptive logic synthesis dataset generation framework designed to support a wide range of machine learning tasks within logic synthesis.
The proposed framework covers the three fundamental stages of logic synthesis: Boolean representation, Logic optimization, and Technology mapping.
The comprehensive workflow includes seven distinct steps, including the raw file generation and the dataset packing.
OpenLS-DGF preserves all original information in the intermediate files, which are stored in both Verilog and ML-friendly GraphML formats.
The verilog files offer semi-customization capabilities, enabling researchers to integrate desired intermediate steps and utilize previously generated verilog files.
Furthermore, OpenLS-DGF includes a specialized circuit engine, which was developed to facilitate effective dataset packaging and extraction of adaptive sub-datasets for multiple tasks.
This circuit engine can faithfully reconstruct the original Boolean circuit information, enabling a wide range of operations to be directly applied for further processing.

We generate the OpenLS-D-v1 dataset utilizing the above framework to facilitate multiple machine-learning tasks.
OpenLS-D-v1 starts from 46 combinational designs from well-established benchmarks~\cite{IWLS2005, IWLS2015, opencores}, including a diverse circuit type, such as arithmetic circuits, control circuits, and IP cores.
It encompasses more than 966,000 Boolean circuits, each derived from 1,000 unique synthesis recipes.
The breakdown of the dataset includes 7000 Boolean networks across 7 logic types, alongside 7000 ASIC and 7000 FPGA netlists.
Moreover, QoRs are preserved in JSON format alongside their corresponding Boolean circuits.
To showcase the versatility of OpenLS-D-v1, we have implemented and tested four typical machine-learning tasks within logic synthesis: circuit classification, circuit ranking, QoR prediction, and probability prediction.
Each task explores unique processes of logic synthesis, employing datasets that are directly extracted and specifically reformatted from OpenLS-D-v1.
The experimental results substantiate the diversity and effectiveness of OpenLS-D-v1, confirming its value across various machine-learning applications and demonstrating the prominent diversity and applicability of OpenLS-DGF.

The contributions can be summarised as follows:
\begin{itemize}
    \item We introduced OpenLS-DGF, an adaptive logic synthesis dataset generation framework that covers three pivotal stages: Boolean representation, logic optimization, and technology mapping. OpenLS-DGF also offers semi-customized capabilities, allowing the reuse of intermediate files for researchers to integrate additional steps as needed;
    \item We developed an adaptive circuit engine capable of loading multiple types of Boolean circuits without losing any information. This engine serves as a bridge between the framework and various downstream tasks, facilitating the extension of operations to generate desired features;
    \item We generated OpenLS-D-v1, an adaptive logic synthesis dataset generated using OpenLS-DGF, designed to support various machine learning tasks. This ensures that all feasible downstream tasks can derive their specific datasets directly from OpenLS-D-v1;
    \item We implemented four typical downstream tasks utilizing the OpenLS-D-v1 dataset to demonstrate its diversity and effectiveness. Moreover, the circuit ranking is a novel task in logic synthesis introduced by this study, highlighting improvements in technology mapping.
\end{itemize}

This paper is structured as follows:
\cref{sec:background} provides the background and related works;
\cref{sec:flow} presents the details of OpenLS-DGF;
\cref{sec:dataset} introduce the generated OpenLS-D-v1 dataset and its key characteristics;
\cref{sec:tasks} formulates the selected four downstream tasks on OpenLS-D-v1 and gives the experimental results;
\cref{sec:discussion} gives the discussion, and \cref{sec:conclusion} draws the conclusion.

%%%%%%%%%%%%%%%%%%%%%%%%%%%%%%%%%%%%%%%%%%%%%%%%%%%%%%%%%%%%%%%%%%%%%%%%%%%%
%   background
%%%%%%%%%%%%%%%%%%%%%%%%%%%%%%%%%%%%%%%%%%%%%%%%%%%%%%%%%%%%%%%%%%%%%%%%%%%%
\section{Background and Related Work}
\label{sec:background}

\begin{figure}
    \centering
    \includegraphics[width=1\linewidth]{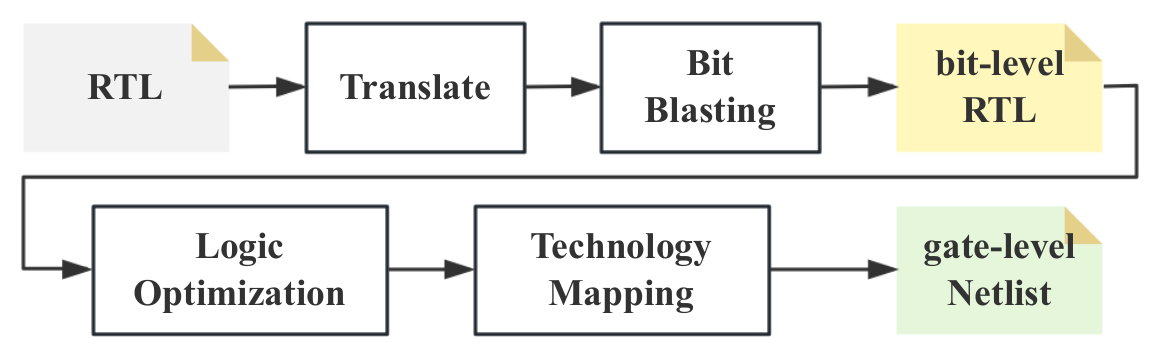}
\vspace{-0.3cm}
    \caption{The Logic Synthesis flow.}
    \label{fig:eda_flow}
\vspace{-0.3cm}
\end{figure}

\cref{fig:eda_flow} illustrates the essential steps of the logic synthesis flow.
The following subsections will introduce the fundamental concepts and related works of logic synthesis.

\subsection{Background}

\subsubsection{Boolean circuit and Functional completeness}

\newcolumntype{M}[1]{>{\centering\arraybackslash}m{#1}}
\begin{table}[t]
\centering
\scriptsize
\caption{The logic gates pool in this framework.}
\begin{tabular}{M{1.4cm}|M{2.6cm}|M{3cm}}
\toprule
\textbf{Logic Gate} & \textbf{Boolean Expression ~~~f(A, B, C)} & \textbf{Distinctive shape} \\
\midrule
Primitive Gates\footnotemark[1] & ... & ... \\
\midrule
MAJ3   & $A \cdot B + A \cdot C + B \cdot C$ & \includegraphics[width=1.5cm]{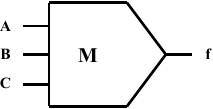} \\
NAND3  & $\overline{A \cdot B \cdot C}$      & \includegraphics[width=1.5cm]{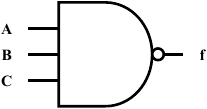} \\
NOR3   & $\overline{A + B + C}$              & \includegraphics[width=1.5cm]{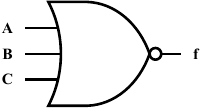} \\
MUX21  & $C \cdot B + \overline{C}\cdot A$   & \includegraphics[width=1.5cm]{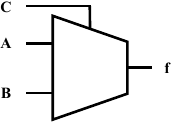} \\
NMUX21 & $C \cdot A + \overline{C}\cdot B$   & \includegraphics[width=1.5cm]{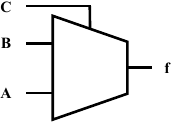} \\
AOI21  & $\overline{A \cdot B + C}$          & \includegraphics[width=2.5cm]{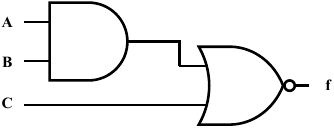} \\
OAI21  & $\overline{(A + B) \cdot C}$        & \includegraphics[width=2.5cm]{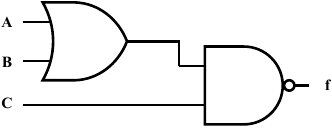} \\ 
\bottomrule
\end{tabular}
\label{tab:cell}
\vspace{-0.2cm}
\end{table}
\footnotetext[1]{Primitive gates are composed of: NOT, BUFFER, AND2, NAND2, OR2, NOR2, XOR2, XNOR2.}

The \textit{\textbf{Boolean circuit}} $\mathcal{C}$ is defined as a computational graph representation with specific Boolean function. 
It can be formulated by the following: $\mathcal{C}= (\mathcal{V}, \mathcal{E}), \mathcal{V} = \mathcal{V}^{PI} \cup \mathcal{V}^{LG} \cup \mathcal{V}^{PO}, (v_i \to v_j) \in \mathcal{E} ~|~ v_i \in \mathcal{V}, v_j \in \mathcal{V},$
where $\mathcal{V}^{PI}$ represents the primary input nodes (PIs), $\mathcal{V}^{PO}$ represents the primary output nodes (POs), and $\mathcal{V}^{LG}$ represents the internal logic gates.
\cref{tab:cell} shows the used logic gates in this paper but the technology-dependent cells.
Each edge $v_i \to v_j$ in $\mathcal{E}$ represents the connected signals between nodes.

Furthermore, the technology-independent Boolean circuits, also referred to as the Boolean network, primarily concentrate on the topology and Boolean function.
On the other hand, the technology-dependent Boolean circuit, which represents a gate-level netlist, incorporates physical attributes such as area, and timing.
It should be noted that the sequential Boolean circuits are not discussed in this work.

\begin{table}[tb]
\centering
\scriptsize
\caption{The related functional complete set.}
\begin{tabular}{c|c}
\toprule
\textbf{Logic Circuit Type}         & \textbf{Functional Complete Set} \\
\midrule
\makecell{And-Inverter Graph\\(\textbf{AIG})}       & NOT, AND2 \\ \hline
\makecell{Or-Inverter Graph\\(\textbf{OIG})}        & NOT, OR2 \\ \hline
\makecell{Xor-And-Inverter Graph\\(\textbf{XAG})}   & NOT, XOR2, AND2 \\ \hline
\makecell{Majority-Inverter Graph\\(\textbf{MIG})}  & NOT, MAJ3 \\ \hline
\makecell{Primitive-Gate Graph\\(\textbf{PRIMARY})} & \makecell{NOT, AND2, NAND2, \\OR2, NOR2, XOR2, XNOR2}  \\ \hline
\makecell{Generic-Technology Graph\\(\textbf{GTG})} & \makecell{\{\textbf{PRIMARY}\}, NAND3, NOR3, \\MUX21, NMUX21, AOI21, OAI21}  \\
\bottomrule
\end{tabular}
\label{tab:fcs}
\vspace{-0.5cm}
\end{table}

\label{sec:background:fcs}
\begin{definition}[Functional completeness]
\label{def:fcs}
A set of logical gates $\mathcal{S}$ is called functionally complete, if for any Boolean function $f$, there exists a circuit using only gates from $\mathcal{S}$ that can represent $f$.
\end{definition}

As defined at \cref{def:fcs}, the Boolean circuit employs a functionally complete set, thereby enabling the representation of any Boolean function through circuit type.
\cref{tab:fcs} illustrates the Boolean networks utilized in this work, including AIG, OIG, XAG, MIG, PRIMARY, and GTG.
While the gate-level netlists are involved after the technology mapping.

\textbf{\textit{Boolean representation task}:}
Different Boolean circuits, based on functional completeness, can exhibit varying performances across different stages of the EDA flow, a phenomenon known as the Boolean representation problem.

\subsubsection{Logic Optimization and Technology Mapping}
The Boolean equivalence~\cite{boolean_equivalence} asserts that the different Boolean circuit graph structure may lead to the same Boolean function.
Moreover, it is the fundamental theory for logic optimization and technology mapping.
The \textit{\textbf{logic optimization}} algorithms aim to reduce the cost of the Boolean network to improve the desired criterion~(area, timing, ...).
Then, the optimized Boolean networks are translated into the gate-level netlists by \textit{\textbf{technology mapping}} with the given standard cell library.
This standard cell library encompasses a functionally complete set of gate-level netlists equipped with essential physical attributes required for technology mapping.

\textbf{\textit{Circuit classification task}:}
According to the Boolean equivalence theory, the Boolean networks derived from the same design have the same functionality.
From this viewpoint, these Boolean networks are in the same class.

\textbf{\textit{Probability prediction task}:}
The functionally equivalent nodes within one Boolean network can be merged to reduce the size.
By computing the probability of nodes, it is possible to effectively identify and filter the functionally equivalent nodes.

\subsubsection{Static Timing Analysis}
Static Timing Analysis~(STA) is a critical technique used to ensure that the gate-level netlist meets specified timing requirements.
STA involves calculating the delay for each signal path within the circuit.
The total delay for any path is determined by the equation:
$$
\textit{path\_delay} = \sum (\textit{gate\_delay}) + \sum (\textit{wire\_delay}),
$$
where $gate\_delay$ represents the internal delay of each gate and $wire\_delay$ accounts for the delay of the wires between gates.
The maximum path delay of the critical path~(arrival time) is typically used to assess if the circuit can operate within the desired timing period.

\textbf{\textit{QoR prediction task}:}
Different logic optimization and technology mapping configurations can lead to different QoR results.
If the QoR distribution is determined, it is possible to predict the related QoR.

\subsubsection{Graph Neural Network~(GNN)}
GNNs are particularly adept at handling data structured in the form of graphs, offering important insights in applications where relationships and interactions are crucial.
A GNN utilizes a multi-layered structure where each layer $K$ updates a node's representation by aggregating features from its neighbors. This aggregation follows the update rule:
\begin{align*}
    h_{\mathcal{N}(v)}^k = & \text{AGGREGATE}_k(\{h_u^{k-1}, \forall{u} \in \mathcal{N}(v)\}),   \\
    h_v^k = & \sigma( \mathbf{W}^k \cdot \text{CONCAT}(h_v^{k-1},h_{\mathcal{N}(v)}^k)).
\end{align*}
where $h_v^k$ represents the embedding of node $v$ in depth $k$; the \text{AGGREGATE} is the differentiable aggregator function; $\mathcal{N}(v)$ represents the neighborhood nodes of node $v$; $\mathbf{W}$ is the weight matrices and $\sigma$ is the non-linearity function.
% This iterative process uses local node features and the overall graph structure to learn comprehensive node representations. These can be utilized for a variety of downstream tasks such as node classification, graph classification, and link prediction, showcasing the capability of GNNs to extract significant insights from complex, interconnected data.

% Power analysis is the process of evaluating the power consumption of a circuit, which is particularly important for low-power design.
% It is composed of static power and dynamic power.
% The static power can be directly computed by its static configuration:
% $ P_{static} = I_{leak} \cdot V. $
% where $I_{leak}$ is the leakage electricity, $V$ is the power voltage.
% As the dynamic power depends on the other configuration of the physical design steps, it is not discussed here.

\subsection{Related Work}
Since the release of the "ImageNet" dataset~\cite{ImageNet}, the field of artificial intelligence has experienced significant advancements in computer vision. ImageNet has enabled a variety of applications, including image classification, segmentation, and detection. This progress has led to the development of innovative ML algorithms that are transforming fields such as autonomous driving, robotics, and natural language processing. 

In recent years, the application of ML in logic synthesis has also seen considerable growth. The "OpenABC-D"~\cite{openabc-d} dataset is generated by the OpenLane~\cite{OpenLANE} flow, which mainly produces intermediate files rather than being specifically designed for dataset generation. Existing datasets, such as those for probability prediction~\cite{DeepGate} and node classification~\cite{Gamora}, are typically tailored for specific tasks. However, there is a notable lack of a comprehensive dataset that can support multiple tasks.

\begin{figure*}[t]
    \centering
    \includegraphics[width=1\linewidth]{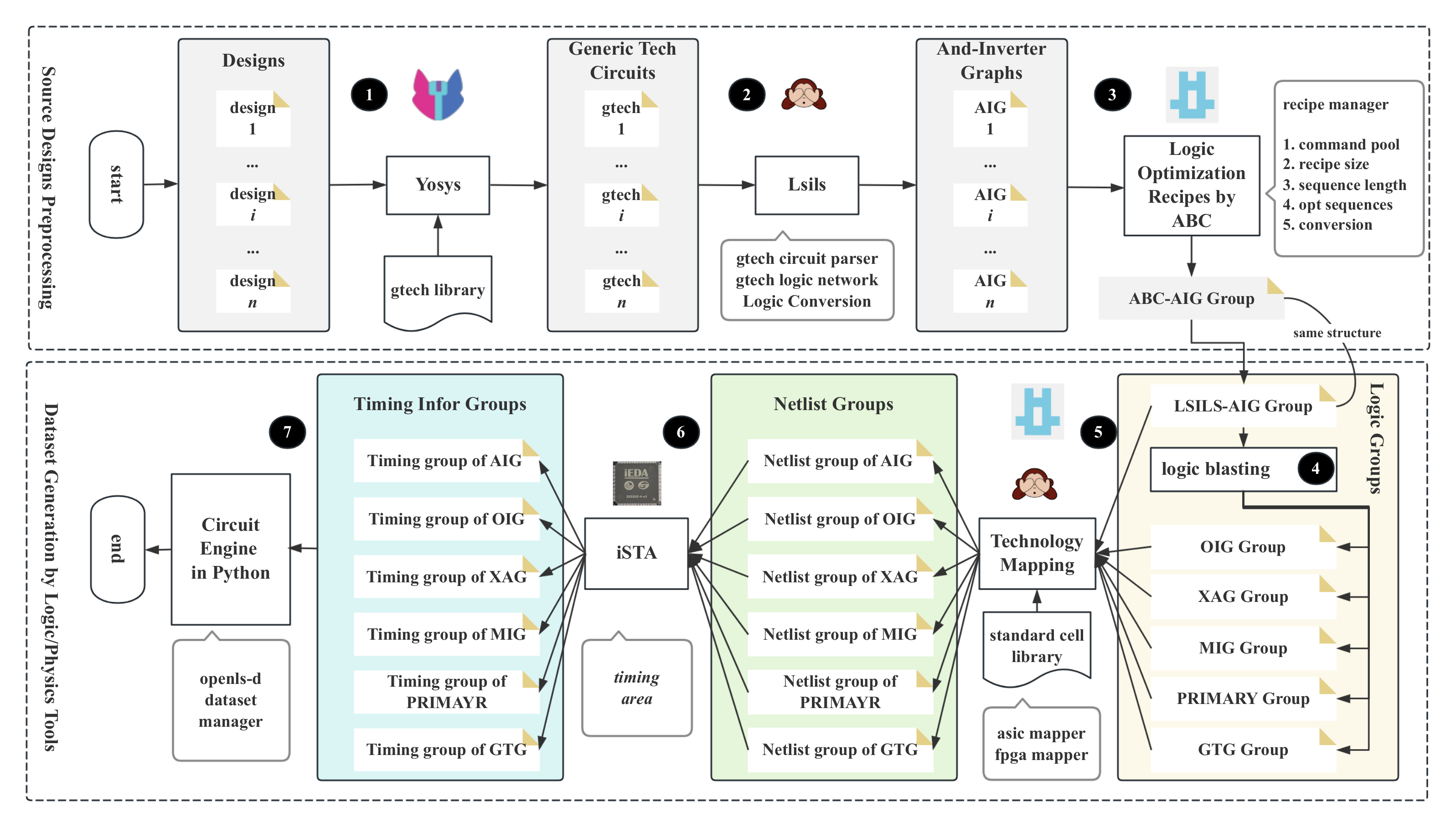}
\vspace{-0.5cm}
    \caption{The adaptive logic synthesis dataset generation framework of OpenLS-DGF.}
    \label{fig:flow}
\vspace{-0.5cm}
\end{figure*}

To address this gap, we introduce OpenLS-DGF, the logic synthesis dataset generation framework designed for various downstream tasks. 
OpenLS-DGF not only accommodates multiple tasks but also enables these tasks to share a common dataset. 
This approach standardizes the measurement of task performance, promoting consistency across different evaluations and enhancing comparability within the field.

%%%%%%%%%%%%%%%%%%%%%%%%%%%%%%%%%%%%%%%%%%%%%%%%%%%%%%%%%%%%%%%%%%%%%%%%%%%%
%   generation flow
%%%%%%%%%%%%%%%%%%%%%%%%%%%%%%%%%%%%%%%%%%%%%%%%%%%%%%%%%%%%%%%%%%%%%%%%%%%%

\section{OpenLS-DGF}
\label{sec:flow}

In this section, we introduce the proposed OpenLS-DGF along with the circuit engine.

\subsection{Overview}
\cref{fig:flow} illustrates the dataset generation flow of OpenLS-DGF.
It covers the three fundamental steps in logic synthesis: Boolean representation, logic synthesis, and technology mapping.
To streamline the process, all processes are integrated into the open-source platform, LogicFactory~\cite{LogicFactory}, utilizing the TCL command environment.
This framework involves 7 distinct steps, starting from the initial design input to the final dataset packaging.
The first three steps~(1-3) involve preprocessing the input design to generate the generic technology circuit and its optimized AIGs.
The subsequent steps (4-6) are dedicated to producing intermediate Boolean circuits derived from these optimized AIGs, including logic blasting, technology mapping, and physical design.
The final step~(7) packages these Boolean circuits into PyTorch format data using a circuit engine, which facilitates efficient dataset management.
This systematic methodology ensures that the design inputs are processed and transformed into a comprehensive dataset, ready for various logic synthesis applications.
Further details will be demonstrated in the subsequent sections.

\subsection{Dataset Generation Steps}
\label{sec:flow:steps}
\paragraph{Step 1: Generic Technology Circuit Synthesis.}
The first step involves synthesizing the generic technology circuit~(GTG).
Given that the source designs are provided in various formats such as Verilog, AIG, and BLIF, adopting GTG as the standardized representation for these inputs is crucial.
This approach ensures uniform processing across different design types, facilitating more streamlined and consistent handling in subsequent stages of logic synthesis.
We utilize Yosys~\cite{Yosys} served as the frontend parser for these format designs, consequently, the input source designs are translated to GTG using Yosys' ``techmap'' method.
The generic technology cells defined in GTG correspond to the basic functionally complete set of RTL intermediate language~(RTLIL) within Yosys.
Besides the primitive gates, GTG includes several complex gates such as NAND3, MUX21 AOI21, and OAI21, which preserve the coarse-grained attributes similarities of the RTLIL.

Notably, similar generic technology cells are also utilized in commercial logic synthesis tools such as Design Compiler~(DC) for their intermediate representations, which highlights the practical relevance of creating GTG for the input designs.
We document the GTG in both Verilog and GraphML formats, facilitating further processing and exploration in subsequent processes.

\paragraph{Step 2: And-Inverter Graph Generation.}
AIGs, composed solely of AND2 and INVERTER gates, are fundamental to most logic optimization techniques in logic synthesis due to their simplicity and structural directness.
Open-source tools such as ABC~\cite{ABC} and LSILS~\cite{lsils} have implemented numerous logic optimization algorithms on AIG, including \textit{rewrite}, \textit{balance}, \textit{refactor}, \textit{resubstitution}, etc.
To enable these optimizations, we convert the GTG generated in Step 1 into AIG. 
This conversion leverages the GTG parser and conversion method provided by LogicFactory.
Additionally, we document each AIG in binary format alongside its corresponding Verilog and GraphML files.

\paragraph{Step 3: Logic Optimization Recipes.}
Logic optimization aims to minimize the cost of Boolean circuits.
It also generates different structural Boolean circuits with different optimization configurations~(the optimization sequence).
These Boolean circuit' variants can substantially influence the quality of results (QoR) during technology mapping and subsequent physical design processes.
In this step, we utilize ABC to generate diverse structural Boolean circuits from a specific design's AIG.
These Boolean circuits facilitate the exploration of the QoR distribution for a given design, which is crucial for the tasks related to QoR distribution.

We utilize a comprehensive set of optimization commands frequently employed in logic optimization:
\begin{equation}
\nonumber
\left\{
\begin{aligned}
&\textit{balance}, \\
&\textit{rewrite}, \textit{rewrite -l}, \textit{rewrite -z}, \textit{rewrite -l -z}, \\
&\textit{refactor}, \textit{refactor -l}, \textit{refactor -z}, \textit{refactor -l -z}, \\
&\textit{resub}, \textit{resub -l}, \textit{resub -z}, \textit{resub -l -z}.
\end{aligned}
\right.
\end{equation}
Additionally, the heuristic optimization sequence such as \textit{resyn}, \textit{resyn2}, along with sequence exploration tasks like BOILS~\cite{BOILS} and DRILLS~\cite{DRILLS}, are performed based on the above command pool.
We generate 1000 distinct optimization sequences for each input design, with each sequence randomly composed of 10 commands from the command pool.
To ensure equal selection probability, the \textit{balance} command is repeated, appearing four times in the command pool.
These 1000 distinct optimization sequences facilitate the exploration of the different sequences on the same design as well as the same sequence on different designs, enhancing the analysis of learning of their impact on design optimization.
Following this step, we generated 1000 variant AIGs for each design, with each AIG indexed according to its corresponding optimization sequence.
% Additionally, these parameters can be customized.

\paragraph{Step 4: Logic Blasting.}
Logic blasting is a process designed to transform Boolean networks into various formats.
\cref{tab:fcs} shows the 6 logic types of Boolean network, including AIG, OIG, XAG, MIG, PRIMARY, and GTG.
Despite the relative scarcity of optimization algorithms for these circuit types compared to AIG, logic blasting provides an avenue to potentially generate superior gate-level netlists through technology mapping for other logic types.

In this step, we utilize the LSILS tool to execute the logic blasting.
Initially, the AIG groups generated by ABC are translated into corresponding AIG groups using the LSILS tool.
Both AIG groups maintain identical structures for each corresponding item.
Subsequently, each AIG in the LSILS AIG groups is converted into other logic types through the logic blasting method, which covers the AIG by the specific standard cell library.
For example, the PRIMARY circuit consists of primitive gates \{NOT, AND2, NAND2, OR2, NOR2, XOR2, XNOR2\}.
By utilizing standard cells composed of these gates, we can generate the PRIMARY circuit of the corresponding AIG.
The gates with a larger area have a higher priority during the covering process.

The coverage of AIG, OIG, XAG, PRIMARY, and GTG are capable of generating the necessary supergate library during the covering algorithm by technology mapping.
However, the functionally complete set $S^{MIG}$ of the MIG circuit, \{NOT, MAJ3\}, are inadequate for generating a functionally complete supergate library due to the hardness of generating basic \{``and2'', ``inverter''\} set, which makes it complicated to cover the AIG to MIG by technology mapping.
Instead, we employ a topological node-wise conversion method to achieve this conversion.

\begin{theorem}
\label{thm:blasting}
The logic blasting method preserves the dependency relationships of the original circuit.
\end{theorem}
\begin{proof}
The logic blasting step relies on the mapping step, ensuring that nodes between the circuits, both before and after the blasting, can be precisely matched. 
Thus, they retain the same topological structure.
The matched nodes preserve the dependency relationships.
Additionally, each MAJ3 gate can represent an AND gate, and it is still feasible to meet \cref{thm:blasting}.
\end{proof}

\begin{lemma}
We can equip the logic optimization capability of AIG to the other logic types through logic blasting.
\end{lemma}
Since most logic optimization algorithms are implemented on AIGs. 
According to \cref{thm:blasting}, we can generate similar Boolean circuits by logic blasting on AIGs.
In this manner, the logic optimization capabilities of AIG are extended to other types of Boolean circuits.
Following this step, we are able to generate corresponding groups of AIG, OIG, XAG, MIG, PRIMARY, and GTG Boolean circuits.
All these Boolean circuits are written in Verilog and GraphML file formats.

\paragraph{Step 5: Technology Mapping.}
For each of the 6 Boolean network groups, we generated a gate-level netlist using the same technology mapping algorithm provided by LSILS~\cite{lsils}, specifically through its ``mapper\_asic'' and ``mapper\_fpga'' methods, which are based on their respective template operation.
For ASIC technology mapping, we employed the sky130~\cite{skywater130} standard cell library. 
Similarly, FPGA technology mapping was constrained to the LUT6 cell configuration.
Given that certain tasks are exclusively relevant to ABC AIG, we will also employ ABC's technology mapping algorithms for these specific instances.
For ASIC technology mapping, we use the ``amap'' command, and for FPGA technology mapping, we apply the ``if -K 6'' command to accommodate the requirements.
All resulting gate-level netlists, whether for ASIC or FPGA, are subsequently saved in both Verilog and GraphML file formats to ensure broad compatibility and facilitate downstream applications.

\paragraph{Step 6: Static Timing Analysis.}
The primary goal of logic synthesis is to produce a better gate-level netlist that enhances the subsequent physical design steps.
Assessing the performance of the existing Boolean circuit is crucial, as the timing information significantly influences the gate-level netlist's performance and serves as a key metric for evaluating the results of logic synthesis.
In this step, we utilize the static timing analysis~(STA) tool provided by the open-source physical design tool, iEDA~\cite{iEDA}, to compute the arrival time information for specific ASIC gate-level netlists.
For FPGA netlists, the depth of the circuit provides a precise indicator for timing evaluation, offering a dependable metric for assessing the performance of the output LUT netlist.
All acquired timing information is documented in JSON format, ensuring that it is accessible for further analysis.

\paragraph{Step 7: Dataset Packing}

\begin{figure}[tb]
    \centering
    \includegraphics[width=1\linewidth]{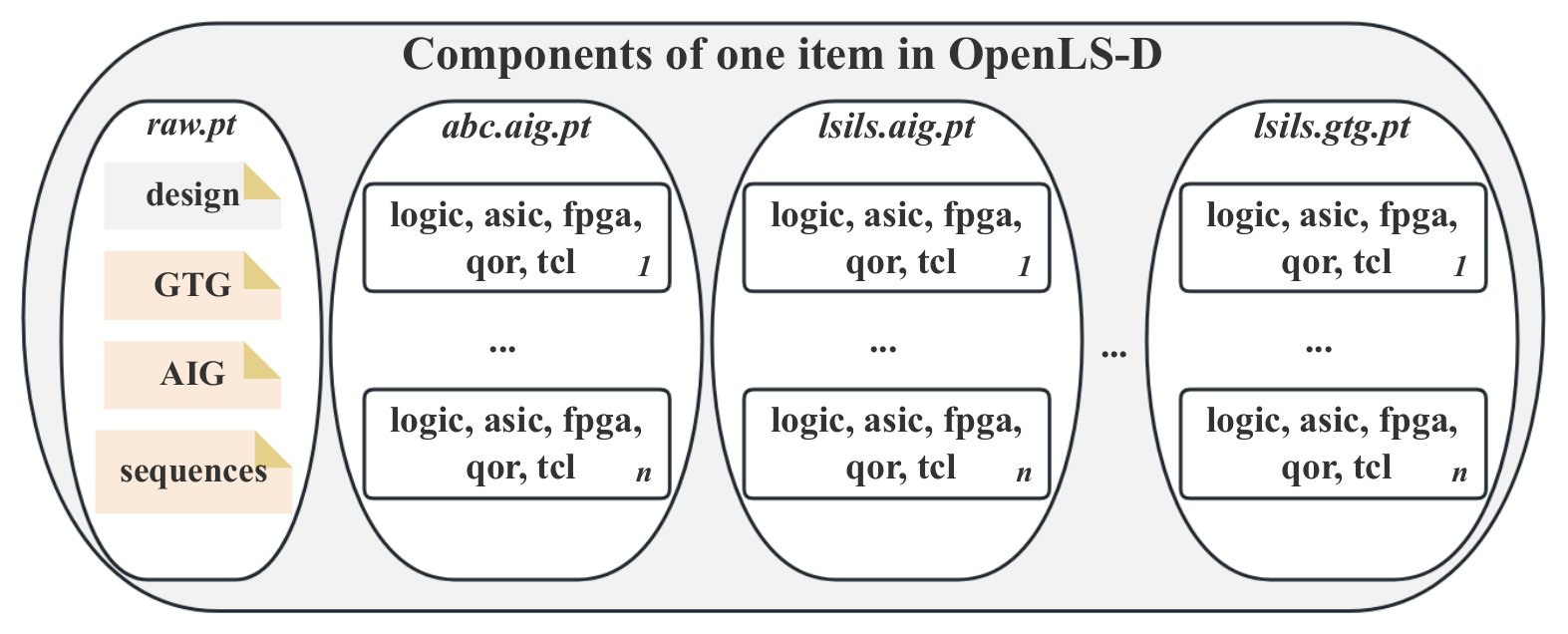}
\vspace{-0.5cm}
    \caption{Components of an item in OpenLS-D-v1. For convenience, the logic types \{OIG, XAG, MIG, PRIMARY\} are abbreviated with ellipses.}
    \label{fig:component}
\vspace{-0.5cm}
\end{figure}

To streamline the management of generated raw files, these files are organized and translated into the PyTorch files by different designs.
This organization is executed based on the circuit engine, which will be discussed further in~\cref{sec:flow:circuit}.

\cref{fig:component} illustrates the components of an individual item in the assembled dataset.
Each item is segmented into 8 PyTorch files: ``\textit{raw.pt}'', ``\textit{abc.aig.pt}'', ``\textit{lsils.aig.pt}'', ``\textit{lsils.oig.pt}'', ``\textit{lsils.xag.pt}'', ``\textit{lsils.mig.pt}'', ``\textit{lsils.primary.pt}'', and ``\textit{lsils.gtg.pt}''. 
The ``\textit{raw.pt}'' file consists of 4 files: the source design, the transformed GTG circuit, the converted AIG circuit, and a fixed set of 1000 optimization sequences.
The ``\textit{abc.aig.pt}'' file contains optimized AIGs generated using the ABC tool with the generated optimization sequence.
Additionally, the ASIC/FPGA gate-level netlists and their corresponding QoR~(timing) are also stored.
The accompanying ``tcl'' file aids in reproducing the respective intermediate files for further processing.
Files such as ``\textit{lsils.aig.pt}'', ``\textit{lsils.oig.pt}'', ``\textit{lsils.xag.pt}'', ``\textit{lsils.mig.pt}'', ``\textit{lsils.primary.pt}'' and ``\textit{lsils.gtg.pt}'' share similar components with the ``\textit{abc.aig.pt}'' file.
However, the Boolean circuits and the technology mapping algorithms they utilize are specifically based on the LSILS tool.

All generated raw files undergo verification using combinational equivalence-checking tools. 
The files within the ``\textit{raw.pt}'' can be checked using Yosys, while the files in the ``\textit{abc.aig.pt}'' are checked by comparing the AIG circuits and their corresponding gate-level netlists through ABC.
Similarly, the remaining files generated by LSILS are checked against their corresponding gate-level netlists in ``\textit{abc.aig.pt}''.

This structured approach to dataset management avoids the creation of excessively large or numerous PyTorch files for any single design.
It allows for selective loading of task-related PyTorch files to derive sub-datasets as needed, enhancing efficiency.
The flexibility of this framework allows for the regeneration of necessary files as required, either before or after certain steps. 
Additionally, the process can be tailored by inserting appropriate steps between the established ones, thus customizing the flow to better meet specific requirements.

\subsection{Circuit Engine}
\label{sec:flow:circuit}

\begin{figure}[t]
    \centering
    \includegraphics[width=1\linewidth]{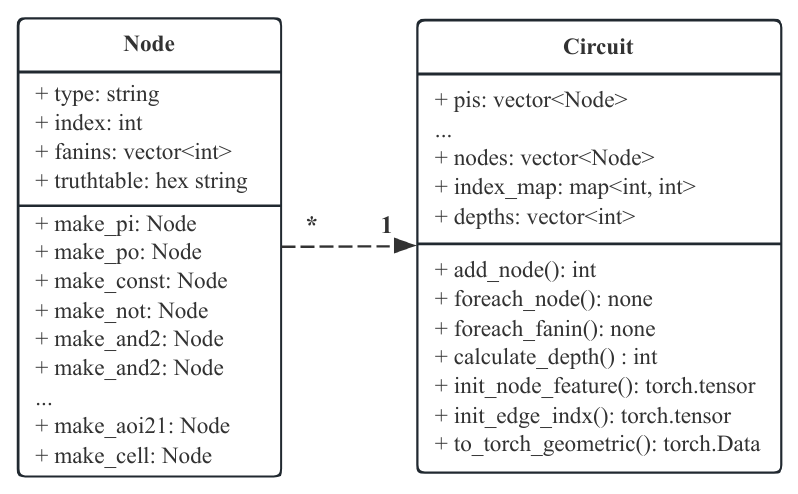}
\vspace{-0.5cm}
    \caption{The UML class diagram of the generic circuit class.}
    \label{fig:circuit}
\vspace{-0.2cm}
\end{figure}

\begin{figure}[t]
    \centering
    \includegraphics[width=1\linewidth]{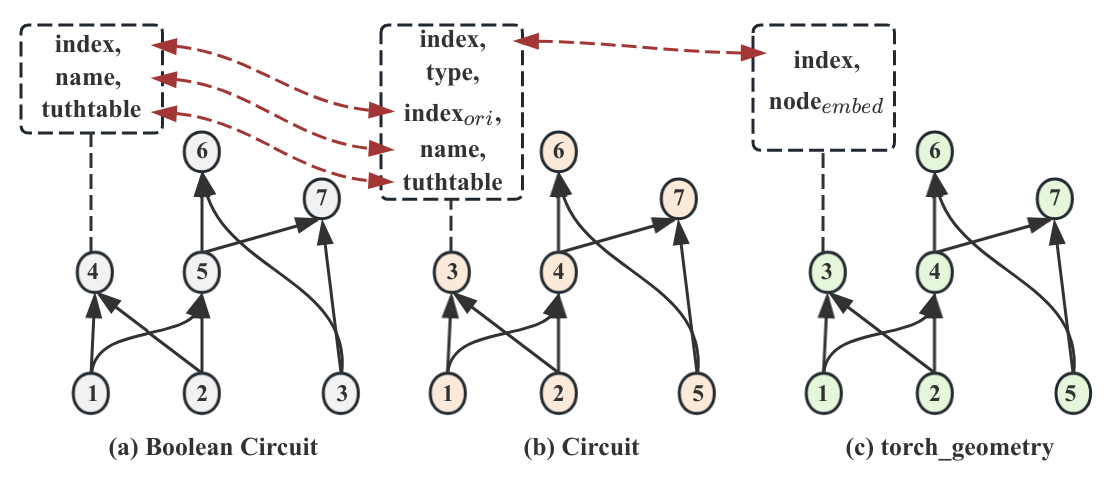}
\vspace{-0.5cm}
    \caption{Node correspondence between the three types of graph: Boolean circuit, Circuit, and torch\_geometry.}
    \label{fig:node_corr}
\vspace{-0.2cm}
\end{figure}

As mentioned in step~(7) in \cref{sec:flow:steps}, the raw files generated are subsequently packaged into the dataset through the proposed circuit engine.
It is primarily comprised of two main parts: the ``Circuit'' class, and operations for ``Circuit''.
\cref{fig:circuit} illustrates the UML diagram of the ``Circuit'' class within the circuit engine.
Each node in the Circuit class has 5 attributes: \textit{type}, \textit{name}, \textit{index}, \textit{fanins}, and \textit{truth table}.
For logic type circuits as listed in \cref{tab:fcs}, the node's type and name correspond directly to their respective gate names such as ``AND2'', ``AOI21'';
whereas for ASIC/FPGA gate-level netlists, the node type is designated as ``CELL'', with the node name corresponding to the matched standard cell name.
The truth table is standardized to 64 bits to accommodate the function of each gate.
Operations on the ``Circuit'' class include operators like ``\textit{to\_torch\_geometry(circuit: Circuit)}'', ``\textit{simulate(circuit: Circuit)}'', etc.
The ``\textit{to\_torch\_geometry(circuit: Circuit)}'' operator ensures consistency in node indices with the Circuit class, facilitating various operations that leverage the ``\textit{torch\_geometry}~\cite{torch_geometry}`` class.
Meanwhile, the ``\textit{simulate(circuit: Circuit)}'' operator allows for the simulation of the circuit using a range of input test cases, thereby providing valuable insights into the circuit's behavior.

The aforementioned Boolean circuits are stored in two formats: Verilog and GraphML.
We use the ``\textit{load\_graphml(path:str)$\to$Circuit}'' operator to load the generated GraphML files into the ``Circuit'' class.
The GraphML files are firstly loaded into a NetworkX~\cite{networkx} graph, and the circuit is constructed by traversing the nodes and edges within this graph.
Since the nodes and edges in the NetworkX graph are stored separately, we establish connected signals between nodes using the ``\textit{add\_fain(fanins)}'' function of the Circuit class during the edge traversal after all nodes have been created.
Moreover, users can easily introduce custom operations within the ``Circuit'' class to meet specific requirements, further utilizing this circuit engine.

\cref{fig:node_corr} illustrates the node correspondence among the three types of graphs: the Boolean circuits, the proposed ``Circuit'', and the ``torch\_geometry'' graph.
In this framework, the Circuit'' graph acts as a pivotal bridge, maintaining the original index ($index_{ori}$) from the input Boolean circuit while assigning a new index ($index$) for internal management. The ``\textit{to\_torch\_geometry(circuit: Circuit)}'' operator ensures that the transformed \textit{torch\_geometry} graph retains the same node indices as the Circuit'' class. 
This design enables smooth interaction between the Boolean circuit structure and the ML-friendly \textit{torch\_geometry} graph, allowing for the efficient transition of features and operations between the two.
Understanding the correspondence between these three types of graphs allows us to leverage their collective strengths to create a dataset with expanded possibilities and enhanced functionality.

%%%%%%%%%%%%%%%%%%%%%%%%%%%%%%%%%%%%%%%%%%%%%%%%%%%%%%%%%%%%%%%%%%%%%%%%%%%%
%   generation flow
%%%%%%%%%%%%%%%%%%%%%%%%%%%%%%%%%%%%%%%%%%%%%%%%%%%%%%%%%%%%%%%%%%%%%%%%%%%%
\section{OpenLS-D-v1}
\label{sec:dataset}

\begin{table}[t]
\centering
\setlength{\tabcolsep}{5pt}
\scriptsize
\caption{Characteristics of the collected designs.}
\begin{tabular}{l|r|r|r|r|r|r}
\toprule
\textbf{Design} & \textbf{\#PI} & \textbf{\#PO} & \textbf{\#And} & \textbf{\#Inv} & \textbf{\#Edge} & \textbf{Depth} \\
\midrule
\textit{adder} & 256 & 129 & 1274 & 1781 & 5226 & 508 \\
\textit{square} & 64 & 128 & 19499 & 24096 & 78151 & 445 \\
\textit{div} & 128 & 128 & 27100 & 37726 & 108555 & 8406 \\
\textit{multiplier} & 128 & 128 & 27753 & 31205 & 111242 & 524 \\
\textit{max} & 512 & 130 & 3021 & 4021 & 12341 & 324 \\
\textit{log2} & 32 & 32 & 32382 & 36027 & 129590 & 597 \\
\textit{sqrt} & 128 & 64 & 32599 & 45647 & 130518 & 10384 \\
\textit{sin} & 24 & 25 & 6604 & 7223 & 26446 & 273 \\
\textit{bar} & 135 & 128 & 2891 & 3468 & 11820 & 26 \\
\textit{cavlc} & 10 & 11 & 652 & 762 & 2623 & 26 \\
\textit{int2float} & 11 & 7 & 208 & 242 & 845 & 23 \\
\textit{i2c} & 177 & 128 & 994 & 1020 & 4148 & 23 \\
\textit{priority} & 128 & 8 & 670 & 845 & 2696 & 384 \\
\textit{voter} & 1001 & 1 & 10528 & 14208 & 42114 & 113 \\
\textit{arbiter} & 256 & 129 & 11923 & 12173 & 47822 & 175 \\
\textit{router} & 60 & 30 & 184 & 173 & 795 & 36 \\
\textit{ctrl} & 7 & 26 & 112 & 119 & 483 & 11 \\
\textit{mem\_ctrl} & 1187 & 962 & 10001 & 10323 & 41691 & 58 \\
\textit{ac97\_ctrl} & 2339 & 2137 & 11129 & 13786 & 48414 & 23 \\
\textit{steppermotordrive} & 28 & 27 & 133 & 145 & 567 & 22 \\
\textit{ss\_pcm} & 104 & 90 & 399 & 476 & 1722 & 14 \\
\textit{usb\_phy} & 132 & 90 & 438 & 465 & 1884 & 17 \\
\textit{sasc} & 135 & 125 & 602 & 770 & 2637 & 16 \\
\textit{simple\_spi} & 164 & 132 & 826 & 989 & 3534 & 21 \\
\textit{spi} & 254 & 238 & 3466 & 3672 & 14242 & 55 \\
\textit{wb\_conmax} & 2122 & 2075 & 45354 & 32111 & 185516 & 32 \\
\textit{wb\_dma} & 828 & 702 & 3644 & 4281 & 15742 & 41 \\
\textit{fir} & 410 & 351 & 4134 & 5628 & 17022 & 159 \\
\textit{des3\_area} & 303 & 64 & 4862 & 4147 & 19544 & 49 \\
\textit{iir} & 494 & 441 & 6302 & 8777 & 25813 & 193 \\
\textit{systemcaes} & 927 & 672 & 9961 & 12880 & 41072 & 74 \\
\textit{systemcdes} & 247 & 128 & 2636 & 3109 & 10728 & 46 \\
\textit{usb\_funct} & 1748 & 1556 & 13098 & 14481 & 54800 & 58 \\
\textit{sha256} & 1943 & 1042 & 14677 & 16283 & 59752 & 198 \\
\textit{dynamic}\_node & 2708 & 2575 & 17402 & 22563 & 74297 & 57 \\
\textit{fpu} & 632 & 409 & 27345 & 34160 & 110018 & 1938 \\
\textit{aes} & 683 & 529 & 28655 & 22518 & 115418 & 44 \\
\textit{aes\_secworks} & 3087 & 2604 & 33953 & 36169 & 140730 & 71 \\
\textit{aes\_xcrypt} & 1975 & 1805 & 50426 & 43280 & 205032 & 79 \\
\textit{tinyRocket} & 4561 & 4181 & 41800 & 50326 & 174934 & 157 \\
\textit{tv80} & 636 & 361 & 9066 & 9333 & 36912 & 109 \\
\textit{ethernet} & 10731 & 10422 & 65509 & 84899 & 282487 & 55 \\
\textit{picosoc} & 11302 & 10797 & 71472 & 82817 & 306788 & 133 \\
\textit{bp\_be} & 11592 & 8413 & 75576 & 97621 & 317662 & 276 \\
\textit{vga\_lcd} & 17322 & 17063 & 103913 & 133019 & 449234 & 41 \\
\textit{jpeg} & 4962 & 4789 & 119908 & 152637 & 488364 & 98 \\
\midrule
\textbf{AVE} &  1882 &  1652 & 20762  & 24400  & 86129  & 574  \\
\bottomrule
\end{tabular}
\label{tab:designs}
\vspace{-0.5cm}
\end{table}

In this section, we will provide a detailed overview of the structure and characteristics of the OpenLS-D-v1 dataset.

\subsection{Data Source}
\subsubsection{Design Selection}
\cref{tab:designs} shows the source designs for OpenLS-D-v1 dataset generation.
These designs are selected from established benchmarks, like IWLS2005~\cite{IWLS2005}, IWLS2015~\cite{IWLS2015}, and OpenCores~\cite{opencores}.
All the presented designs are combinational AIG, with the number of Primary Inputs (PIs) ranging from 7 to 17322, Primary Outputs (POs) from 1 to 17063, AND gates from 112 to 119908, Inverter gates from 119 to 152637, Edge signals from 483 to 488364, and depths from 11 to 10384.

It contains diverse types of designs, including arithmetic circuits, control circuits, and IP cores, making it a comprehensive resource for testing and benchmarking logic synthesis algorithms.
In addition, new designs or internal steps can be added to update the generated dataset for specific demands incrementally.

\subsubsection{Designs Diversity}
\begin{figure}[t]
    \centering
    \includegraphics[width=1\linewidth]{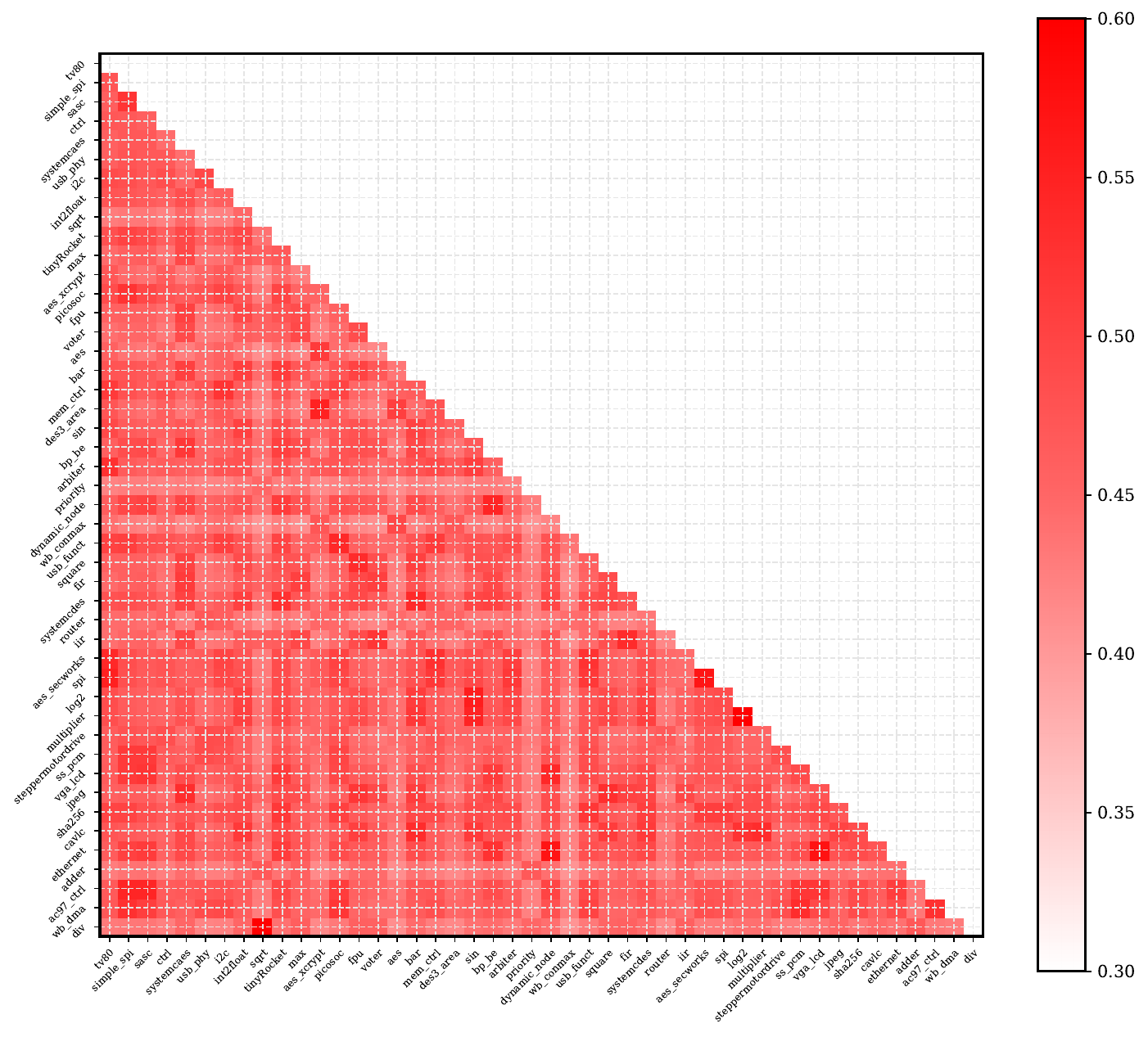}
\vspace{-0.5cm}
    \caption{The graph embedding similarity of the source designs.}
    \label{fig:dataset:similarity}
\vspace{-0.2cm}
\end{figure}

\cref{fig:dataset:similarity} presents the cosine similarity for the source designs as listed in \cref{tab:designs}.
Each design is represented by the graph embedding computed using a combination of the heuristic features and features aggregated from Graph2Vec~\cite{graph2vec}:
\begin{equation}
\nonumber
\begin{aligned}
feature_1 =&~ [pis, pos, ands, invs, edges, depth], \\
feature_2 =&~ \text{graph2vec}(circuit, dimension=128), \\
feature   =&~ \text{concatenate}(feature_1, feature_2).
\end{aligned}
\end{equation}
The heuristic features provide a coarse-grained view of the design, while Graph2Vec provides a deeper insight into the internal structural intricacies.
The combination of these features yields a robust graph embedding for each design.

Cosine similarity calculations are conducted using the ``\textit{sklearn.metrics.pairwise.cosine\_similarity}'' function.
To enhance the visualization of this correlation, the similarity scores ranging from [-1, 1] are normalized to [0.3, 0.6].
Additionally, the diagonal and the top-right corner of the matrix are cleared to eliminate redundancy.
The average cosine similarity score is around 0.44.
This visualization distinctly highlights the variability across different design embeddings.

\subsubsection{Dataset Generation}

\begin{table}[t]
\setlength{\tabcolsep}{1.2pt}
\centering
\caption{The Characteristics comparison between the OpenABC-D~\cite{openabc-d} and OpenLS-D-v1 Datasets.}
\begin{tabular}{c|c|c}
\toprule
\diagbox{Charac}{Dataset}  & OpenABC-D& OpenLS-D-v1~(ours) \\
\midrule
\textbf{source}                 &     OpenCore、IWLS & OpenCore、IWLS、EPFL \\
\midrule
\#\textbf{designs}              &     29            & 46 \\
\midrule
\textbf{raw AIG}                &   $\surd$         & $\surd$ \\
\textbf{raw GTech}              &   $\times$    & $\surd$ \\
\#\textbf{recipes}              &     1500          & 1000 \\
\#\textbf{sequence length}      &     20            & 10   \\
\#$\textbf{AIGs}^{abc}$$\slash$design        &     \textbf{\textit{30000~(1500$\times$20)}}        & 1000 \\
\#$\textbf{AIGs}^{lsils}$$\slash$design      &     $\times$                                        & 1000 \\
\#$\textbf{OIGs}^{lsils}$$\slash$design        &     $\times$      & 1000 \\
\#$\textbf{XAGs}^{lsils}$$\slash$design        &     $\times$      & 1000 \\
\#$\textbf{MIGs}^{lsils}$$\slash$design        &     $\times$      & 1000 \\
\#$\textbf{PRIMARYs}^{lsils}$$\slash$design    &     $\times$      & 1000 \\
\#$\textbf{GTGs}^{lsils}$$\slash$design        &     $\times$      & 1000 \\
\#\textbf{ASIC netlist}$\slash$design       &     $\times$      & 7000 \\
\#\textbf{FPGA netlist}$\slash$design       &     $\times$      & 7000 \\
\midrule
\#\textbf{Total circuits}       &     870k~(30k$\times$29)    & 966k~(21k$\times$46) \\
\bottomrule
\end{tabular}
\label{tab:dataset:comparison}
\vspace{-0.5cm}
\end{table}

\cref{tab:dataset:comparison} provides a detailed comparison between the OpenABC-D and OpenLS-D-v1 datasets. 
The OpenLS-D-v1 dataset includes over 966,000 Boolean circuits, structured into groups where each consists of 21,000 circuits generated from 1,000 distinct synthesis recipes. 
This collection includes 7,000 circuits across 7 Boolean network types, along with 7,000 ASIC and 7,000 FPGA netlists.

To optimize storage efficiency, we employed the ``zstandard'' compression tool~\cite{zstandard}, which significantly reduced the storage footprint. 
The entire dataset generation process, including compression, was executed over approximately 76 hours for raw file generation and 65 hours for compression, using 32 threads on an Intel Xeon Platinum 8380 CPU with a 16T SEAGATE EXOS HDD.
The raw files occupy about 410 GB, while the generated PyTorch files take up about 700 GB.

In comparison, although the OpenABC-D dataset involves generating a larger number of AIGs per design (30,000) from 1,500 recipes, this often results in redundancy within the generated AIGs~(due to the same optimization sequence by sub-sequence extraction).
Conversely, OpenLS-D-v1 adopts a more diverse approach by generating different types of Boolean networks for each design using varied logic types and synthesis sequences.
This method provides a richer and more distinct dataset, better suited for various machine learning applications.

\subsection{Dataset Characteristics}

\begin{figure*}[t]
\centering
\subfigure[\textit{cavlc}]{
\includegraphics[width=0.23\linewidth]{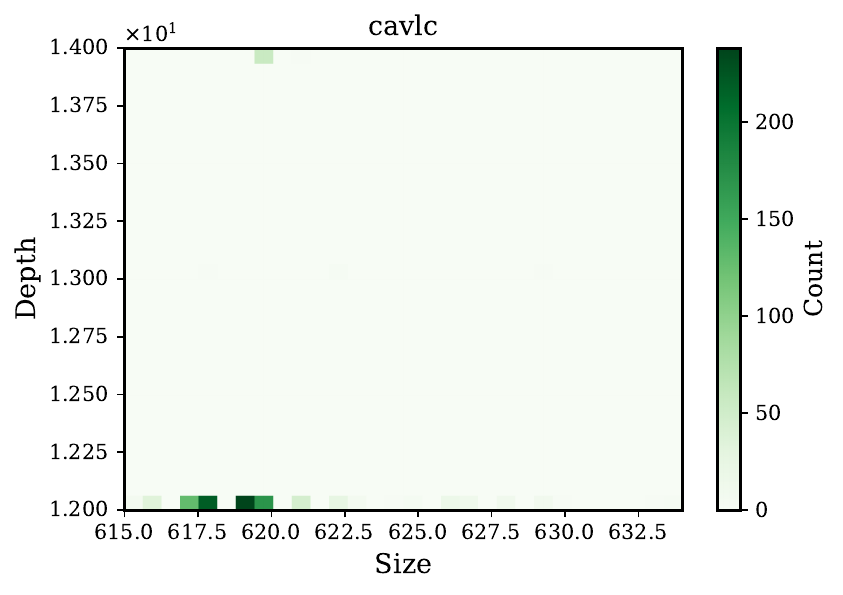}
}
\hfill
\subfigure[\textit{max}]{
\includegraphics[width=0.23\linewidth]{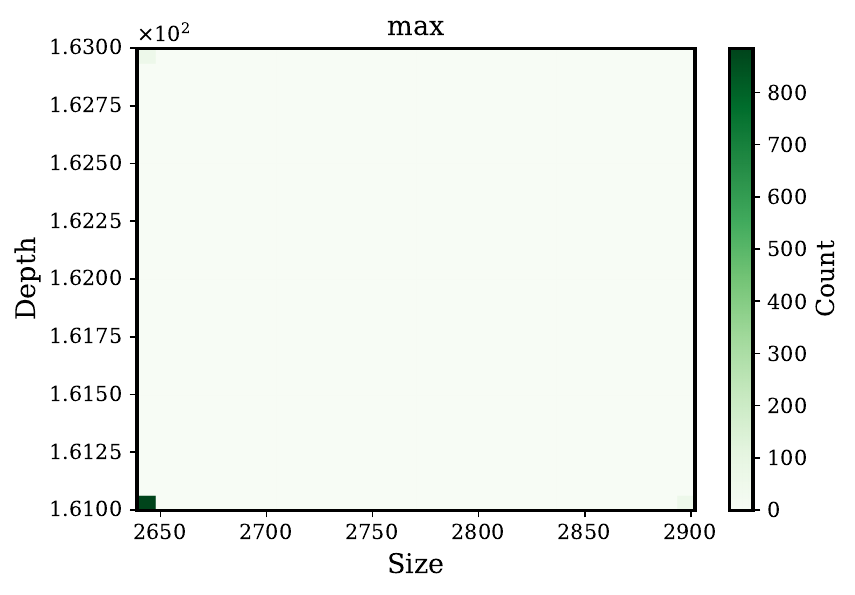}
}
\hfill
\subfigure[\textit{router}]{
\includegraphics[width=0.23\linewidth]{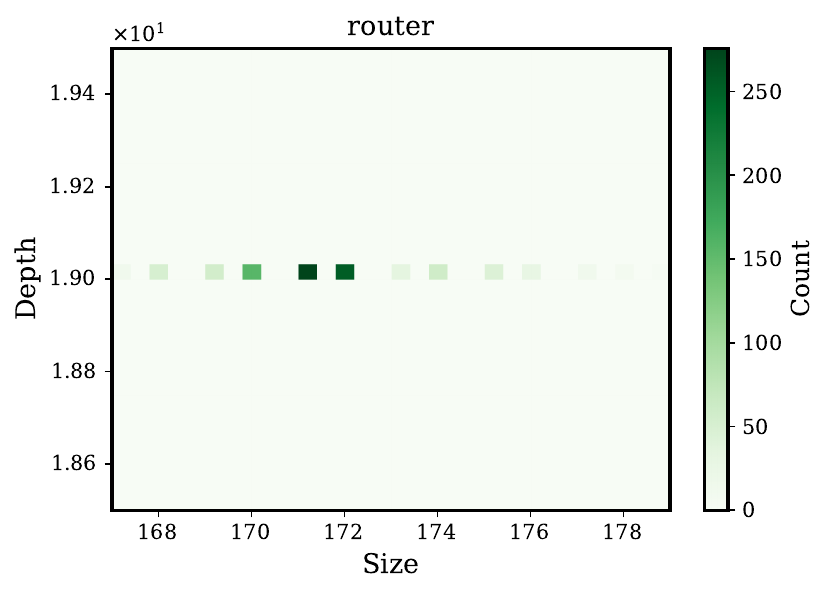}
}
\hfill
\subfigure[\textit{fpu}]{
\includegraphics[width=0.23\linewidth]{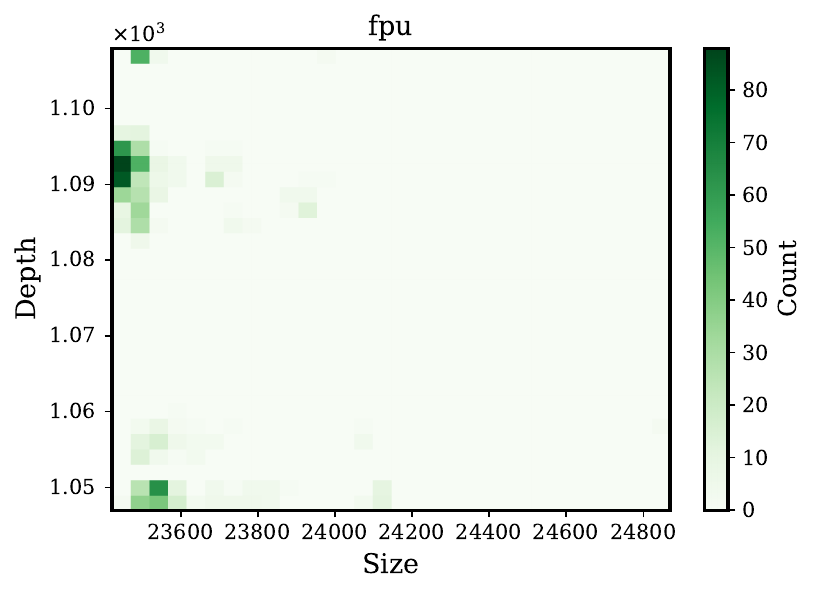}
}
\hfill
\subfigure[\textit{cavlc}]{
\includegraphics[width=0.23\linewidth]{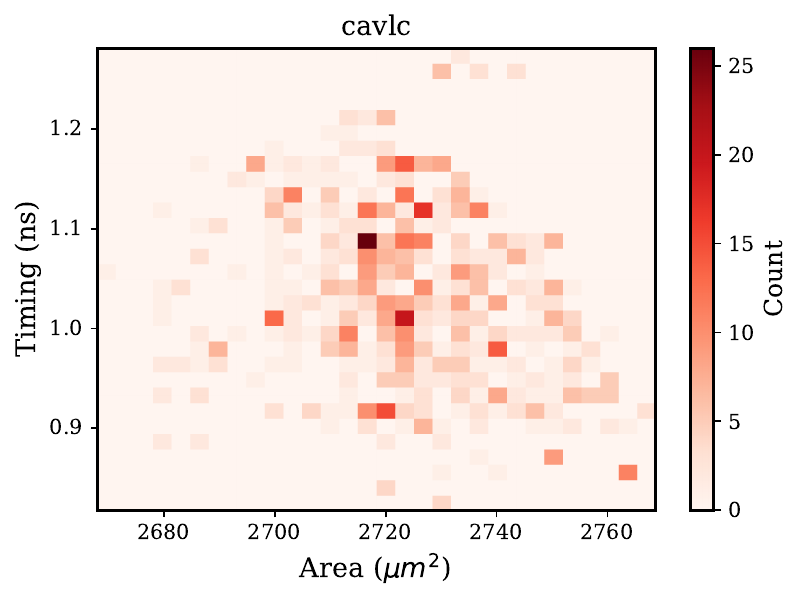}
}
\hfill
\subfigure[\textit{max}]{
\includegraphics[width=0.23\linewidth]{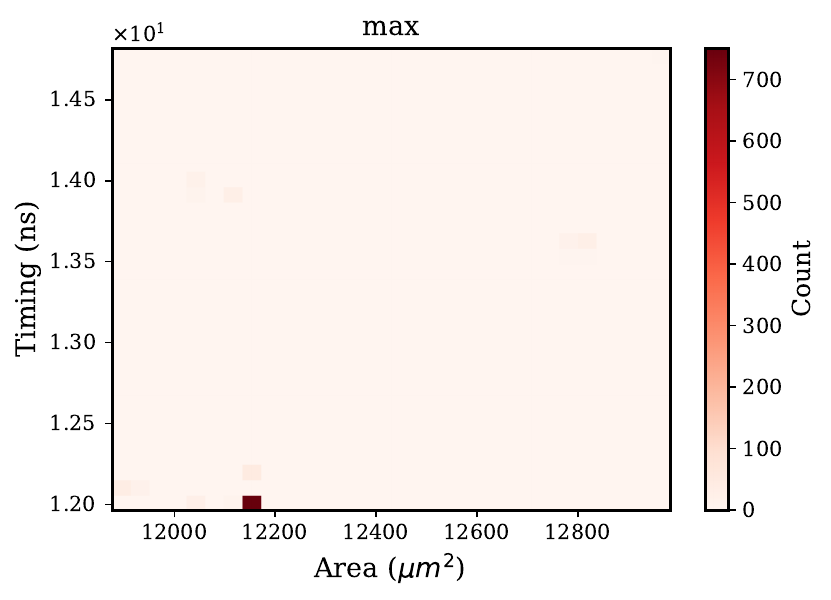}
}
\hfill
\subfigure[\textit{router}]{
\includegraphics[width=0.23\linewidth]{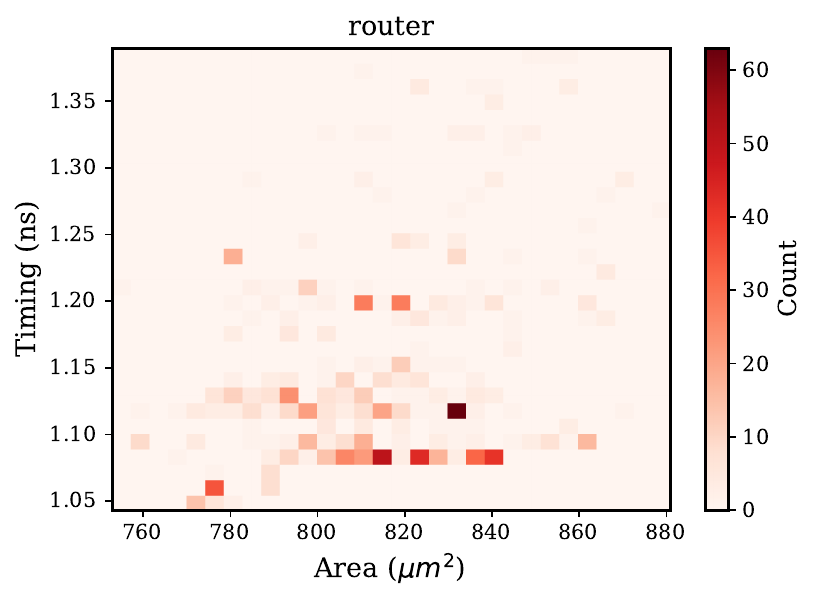}
}
\hfill
\subfigure[\textit{fpu}]{
\includegraphics[width=0.23\linewidth]{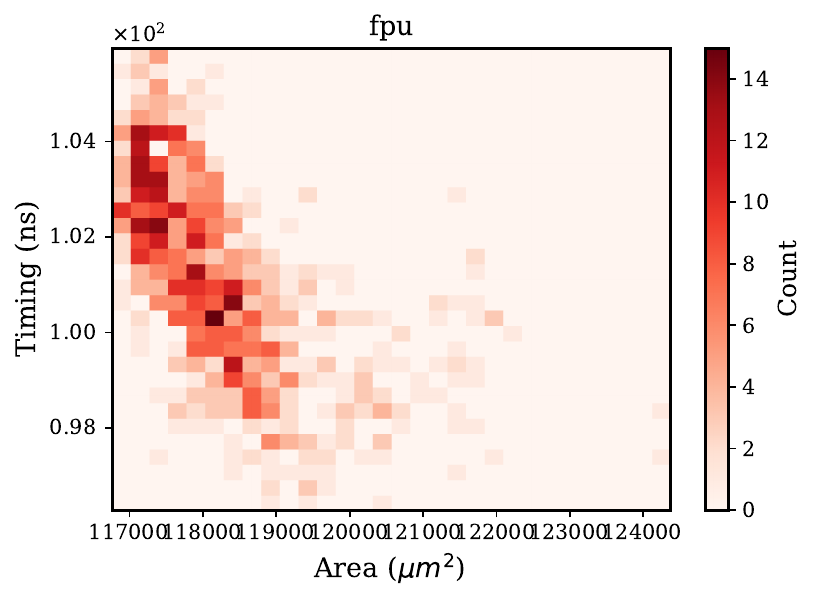}
}
\hfill
\subfigure[\textit{cavlc}]{
\includegraphics[width=0.21\linewidth]{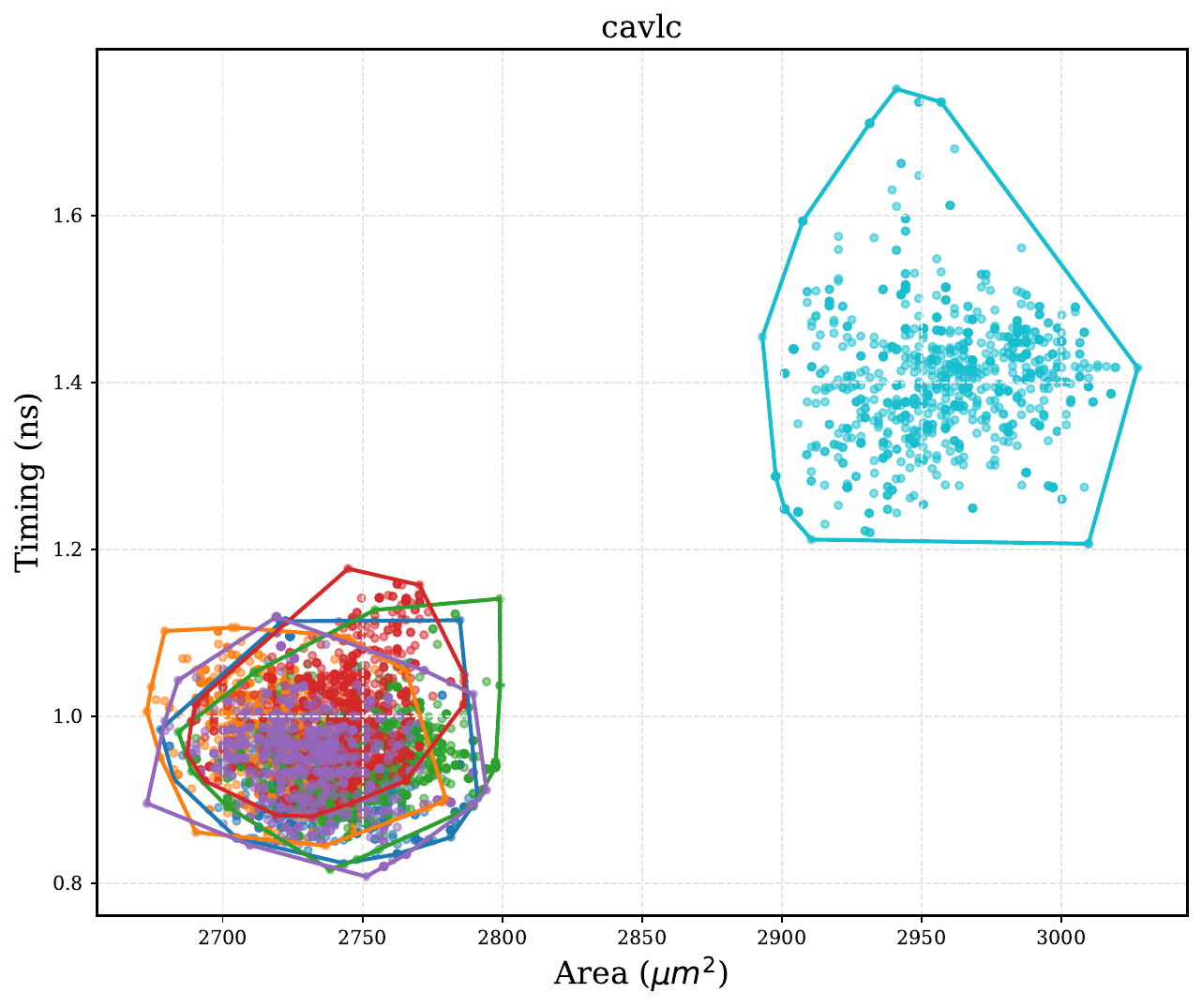}
}
\hfill
\subfigure[\textit{max}]{
\includegraphics[width=0.21\linewidth]{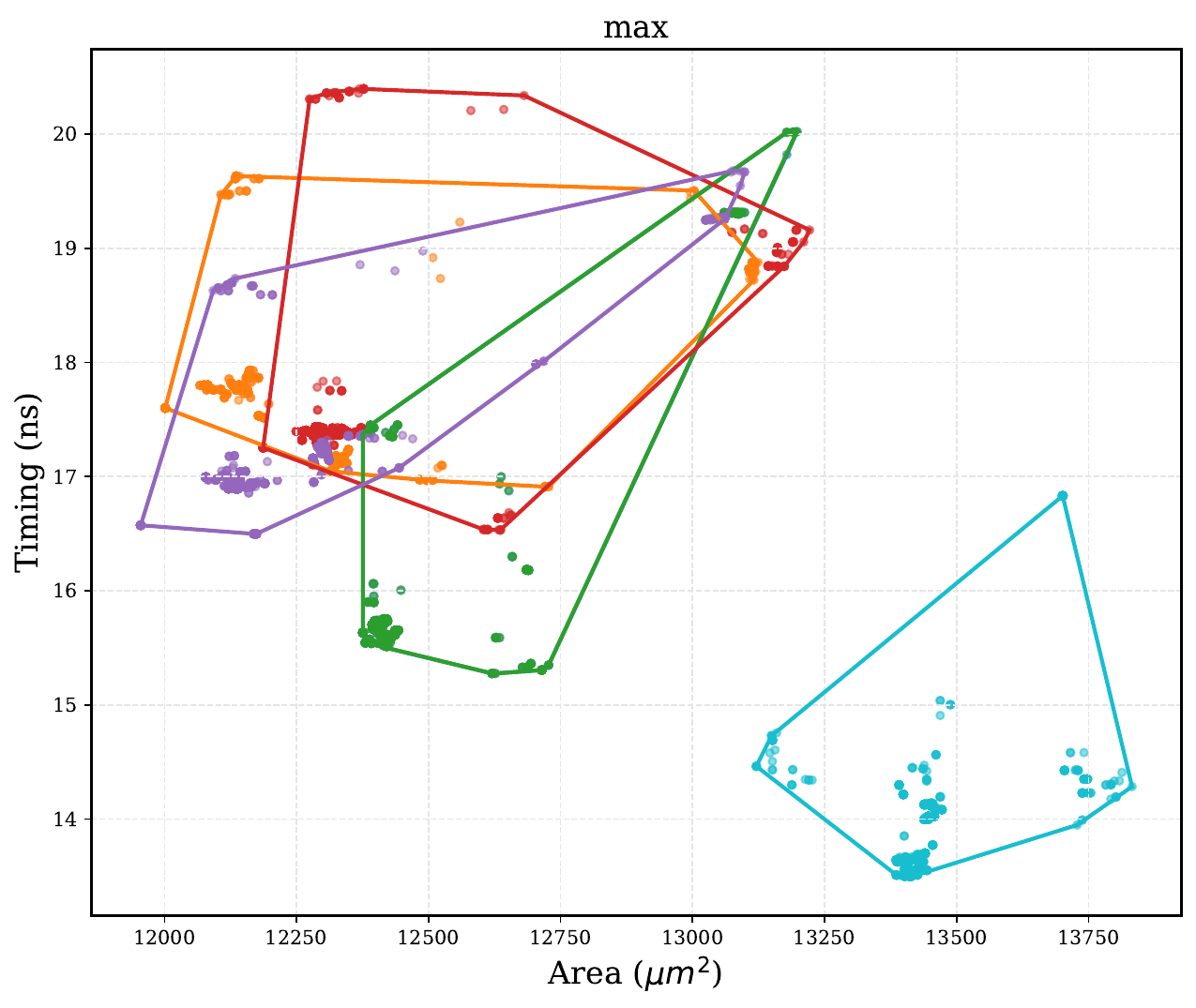}
}
\hfill
\subfigure[\textit{router}]{
\includegraphics[width=0.21\linewidth]{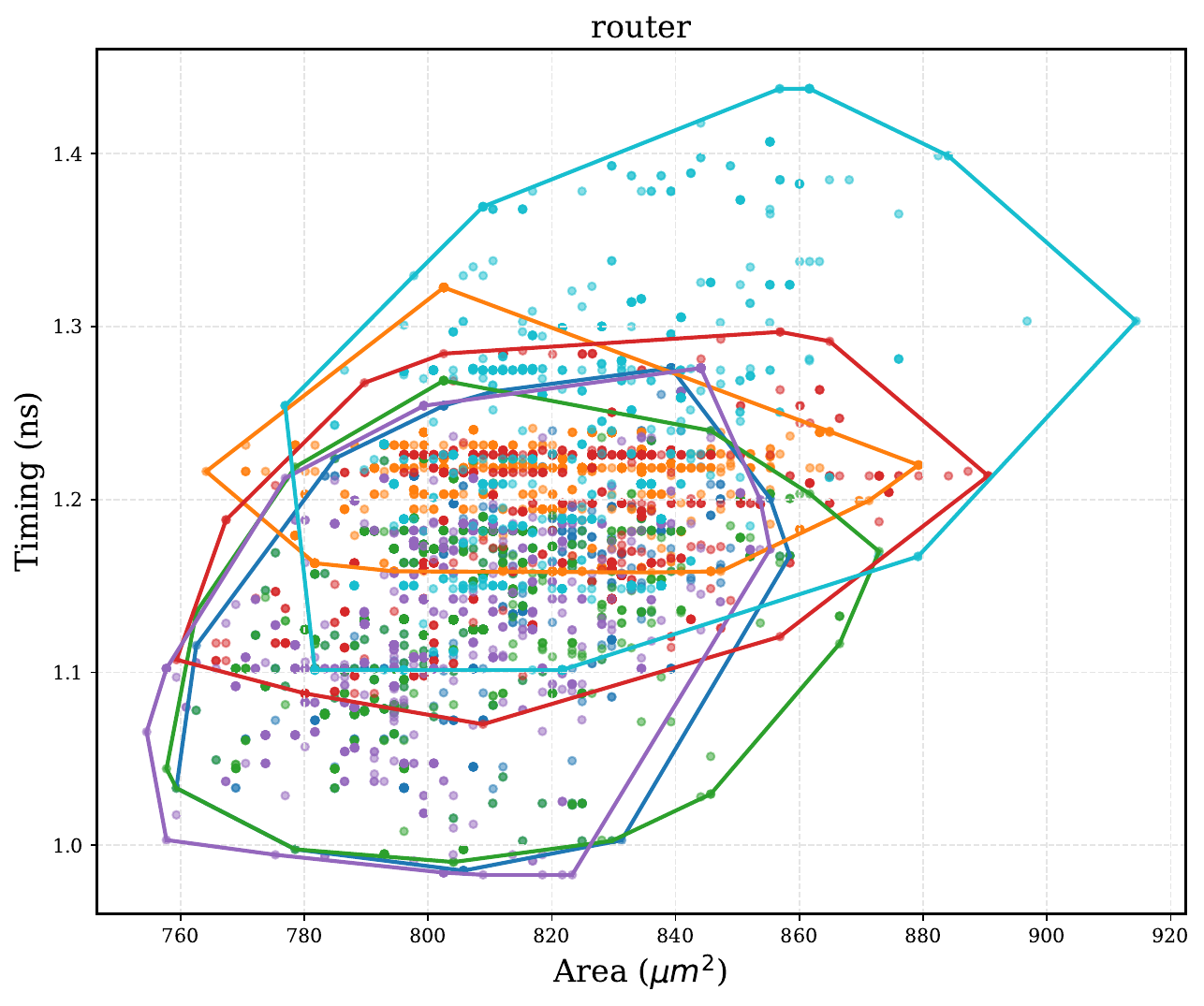}
}
\hfill
\subfigure[\textit{fpu}]{
\includegraphics[width=0.25\linewidth]{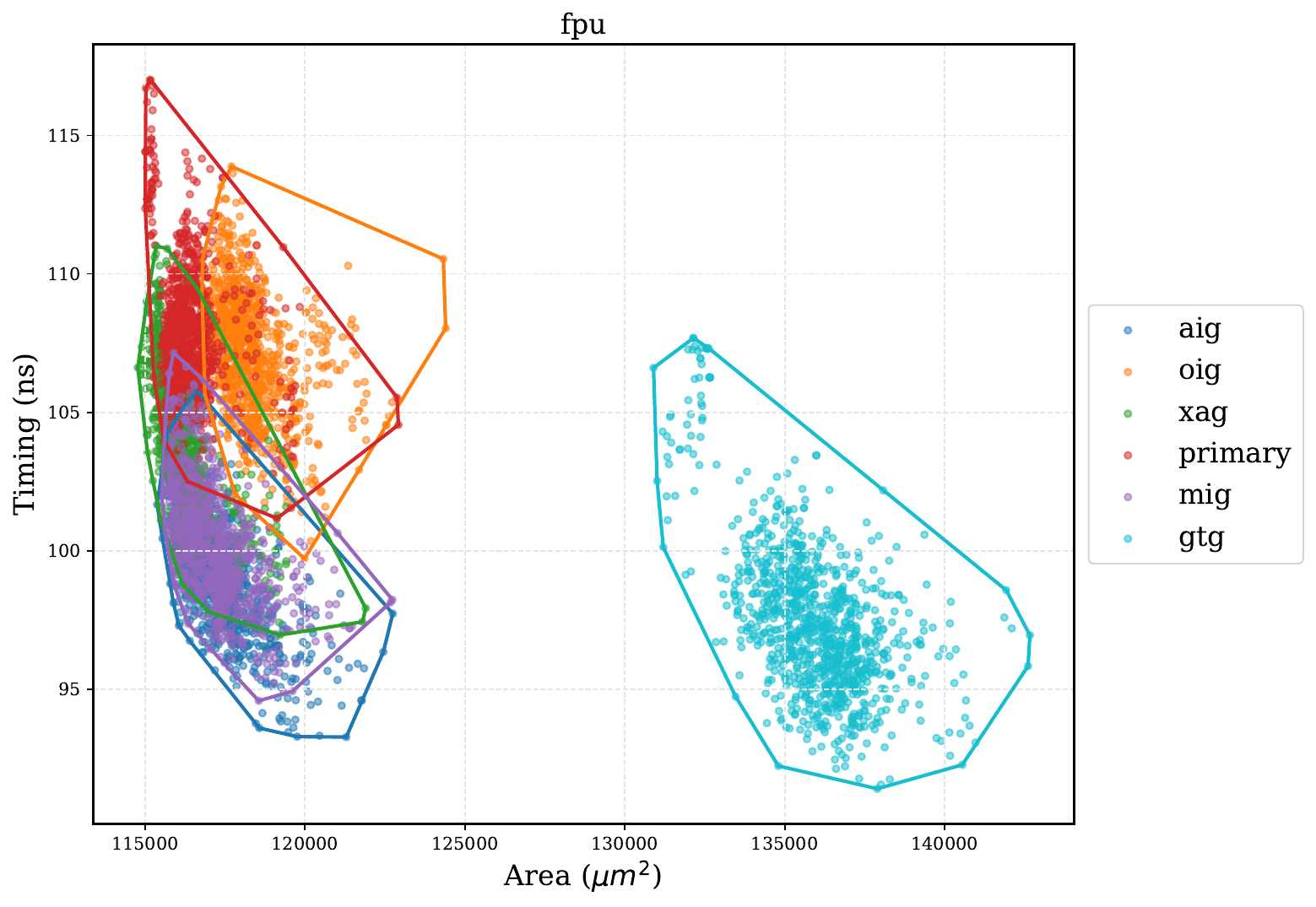}
}
\caption{The QoR distribution for source designs. Each row illustrates the QoR distribution across different designs, and each column presents the three distinct types of QoR measures for a single design. Panels (a) to (d) show the node size and graph depth distribution of the optimization AIGs for the design; (e) to (h) show the area and arrival time distribution for the corresponding ASIC netlists; and (i) to (l) show the convex hulls of the QoR distribution for different types of Boolean networks.}
\label{fig:qor_distr}
\vspace{-0.5cm}
\end{figure*}

\begin{figure*}[t]
\centering
\subfigure[recipe size=250]{
\includegraphics[width=0.23\linewidth]{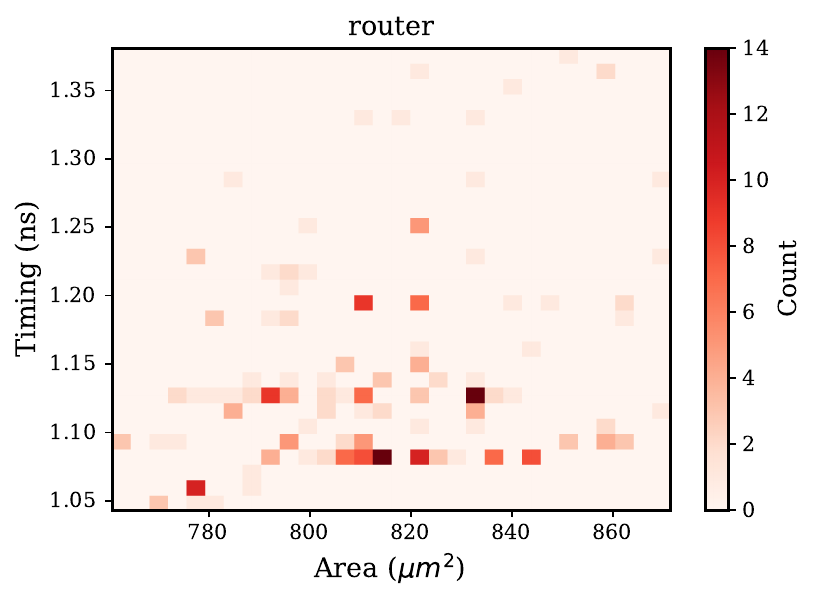}
}
\hfill
\subfigure[recipe size=500]{
\includegraphics[width=0.23\linewidth]{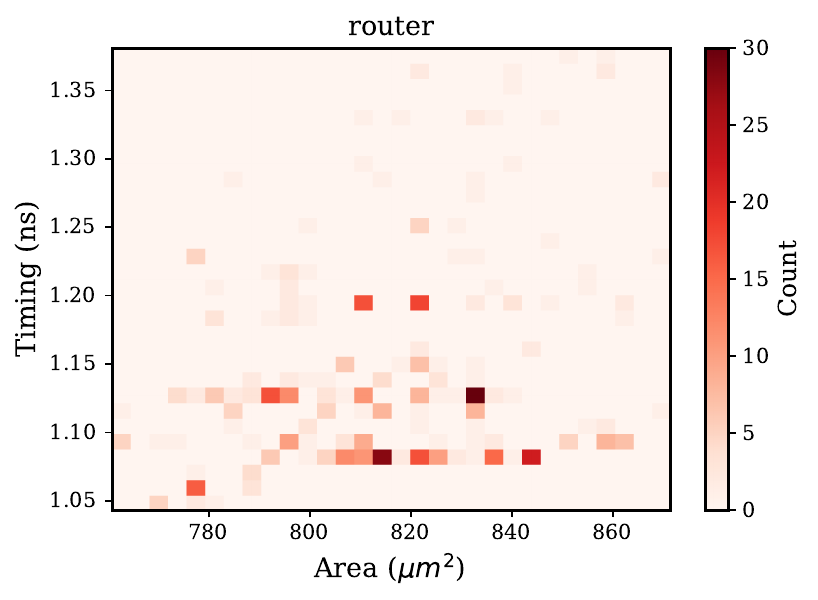}
}
\hfill
\subfigure[recipe size=750]{
\includegraphics[width=0.23\linewidth]{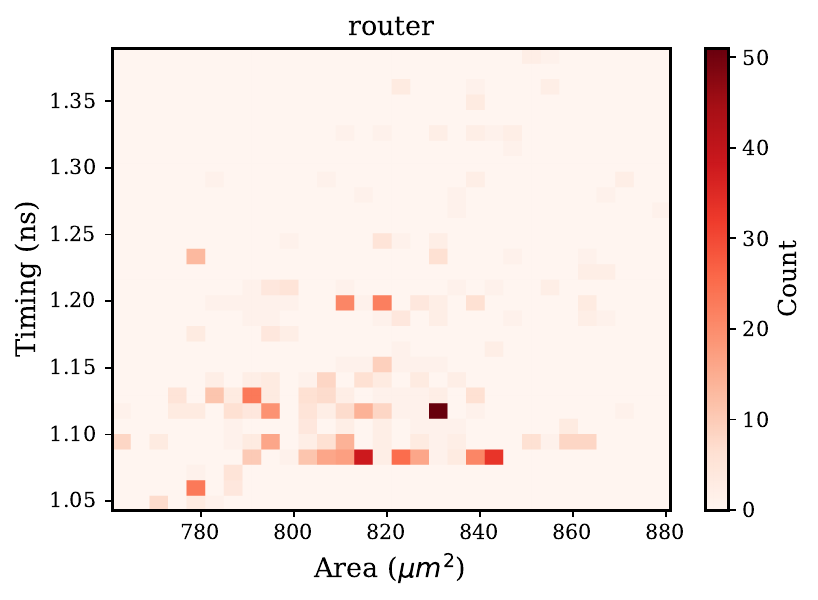}
}
\hfill
\subfigure[recipe size=1000]{
\includegraphics[width=0.23\linewidth]{figures/chara/router_recipe_1000_abc_asic.circuit.pdf}
}
\caption{The distribution of node size and graph depth for one design with incremental recipe size of AIG sets.}
\label{fig:qor_incr_logic}
\vspace{-0.2cm}
\end{figure*}

The QoR within the OpenLS-D-v1 dataset is differentiated into two main categories: technology-independent and technology-dependent metrics.
The technology-independent QoR corresponds to the Boolean circuits' structure, including the node size and maximum graph depth.
Meanwhile, technology-dependent QoR relates to the physical attributes such as the total area and arrival time characteristics.
The technology-independent QoR is applicable for types of Boolean circuits, while the technology-dependent QoR is specific to the gate-level netlists.
Since the FPGA netlist's area and timing are generally related to the size and depth of its graph structure, we just explore the ASIC-specific characteristics here.

\cref{fig:qor_distr} illustrates the QoR distributions for various design parts, highlighting three types of distributions: technology-independent QoR for Boolean networks, technology-dependent QoR for ASIC netlists, and technology-dependent QoR for different logic types of Boolean networks' corresponding ASIC netlists.
From this illustration, the following observations can be drawn:
\paragraph{Observation 1: Boolean networks with the same node size and graph depth can still have different QoR distributions of their corresponding ASIC netlist.}
Logic optimization is more concerned about the local gain that can lead to global gain, it does not necessarily affect the size and depth of the optimized Boolean networks for many optimization operators.
Although there are many similar QoR results of one AIG as shown in \cref{fig:qor_distr} (a) to (d), the area and timing distributions of their corresponding ASIC netlists are significantly dispersed.
This variance is likely due to ASIC technology mapping's sensitivity to the local structures within the AIG.
Thus, it suggests that the QoR distribution of the logic circuit is less indicative of performance compared to that of the corresponding gate-level netlist.

\paragraph{Observation 2: Different Boolean representations of one design may exhibit different behaviors.}
The cut-based technology mapping method is particularly sensitive to the exploration space of the local structures, where different local structures constructed into a circuit can lead to significantly different physical properties.
Consequently, the technology-dependent QoR of gate-level netlist varies among different logic types of Boolean networks with the same index in OpenLS-D-v1 of one design.
\cref{fig:qor_distr} shows the total area and arrival time distribution for the provided logic types of Boolean networks.
It highlights that there are non-overlapping regions between the convex hulls of distributions for different types of Boolean networks.
This indicates that the QoR distributions for different types of Boolean networks are not identical; in some cases, there is no overlap at all between the QoR of different Boolean representations.
Furthermore, the difference between the Boolean networks' QoR distributions also varies across different designs.

\cref{fig:qor_incr_logic} shows the QoR distribution under the incremental recipe size of the logic optimization of one design. 
From this, we can get:
\paragraph{Observation 3: After a certain number of optimized sequences are generated, a QoR interval typically forms, with new optimization sequences likely falling within this range.}
Since logic optimization recipes primarily target AIGs, and other Boolean networks are translated from AIGs through the logic blasting method, the focus can remain on technology-independent QoR distribution for AIGs.
Theoretically, these derived Boolean networks exhibit similar distributions under certain affine transformations.
\cref{fig:qor_incr_logic} not only showcases the node size and graph depth distribution for a selection of source designs but also highlights the incremental QoR changes across optimization sequence milestones at 250, 500, 750, and 1000 optimizations.
This visualization underscores the initial observation that after a certain point, further optimizations tend to fall within a predictable range of QoR.

% \subsection{File System}

% \begin{figure}[h]
%     \centering
%     \includegraphics[width=1\linewidth]{figures/filesystem.pdf}
%     \caption{File system of this work.}
%     \label{fig:filesystem}
% \end{figure}

% The \cref{fig:filesystem} illustrates the organizational structure of the file system for this project. The \texttt{benchmark} folder contains the source data for the designs, while the \texttt{techlib} folder holds technology-dependent files. The \texttt{flow} folder includes scripts for generating raw dataset files, and the \texttt{circuit engine} folder houses the \texttt{Circuit} class along with operations for dataset packaging. Additionally, the \texttt{tasks} directory encompasses code for various downstream tasks. Each task typically comprises three core files:
% \begin{itemize}
%     \item \texttt{subdataset.py}: aids in extracting the required subsets from the full dataset.
%     \item \texttt{net.py}: defines the neural network models.
%     \item \texttt{train.py}: manages the execution of training and testing phases.
% \end{itemize}

% This structure ensures a modular and clear separation of components, facilitating ease of use and further expansion of project functionalities.

%%%%%%%%%%%%%%%%%%%%%%%%%%%%%%%%%%%%%%%%%%%%%%%%%%%%%%%%%%%%%%%%%%%%%%%%%%%%
% downstream tasks
%%%%%%%%%%%%%%%%%%%%%%%%%%%%%%%%%%%%%%%%%%%%%%%%%%%%%%%%%%%%%%%%%%%%%%%%%%%%

\begin{figure}[h]
    \centering
    \includegraphics[width=1\linewidth]{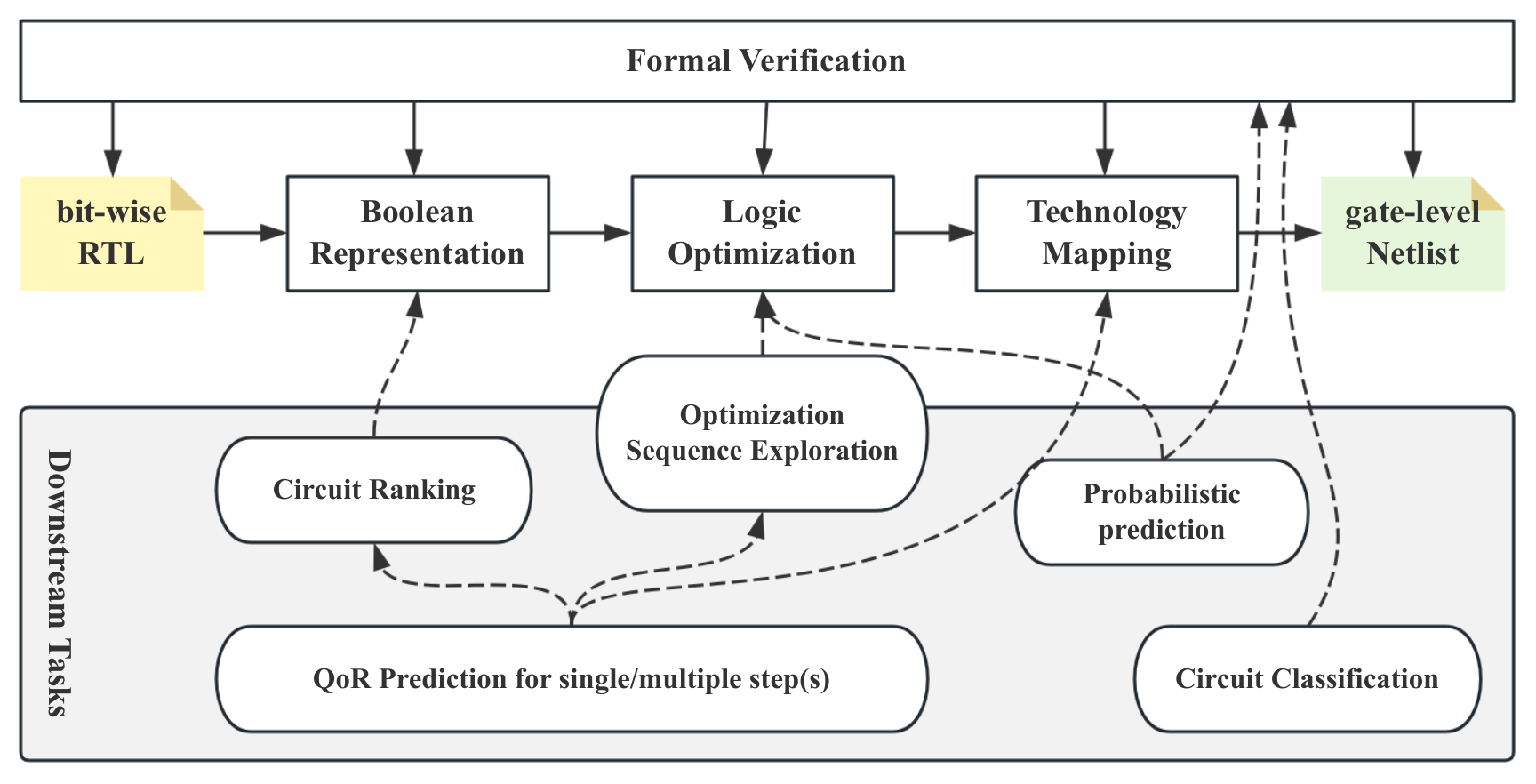}
\vspace{-0.5cm}
    \caption{Illustration of the selected downstream tasks and their potential applications within logic synthesis flow.}
    \label{fig:tasks}
\vspace{-0.5cm}
\end{figure}

\section{Tasks Formulation and Experimental Results}
\label{sec:tasks}

\begin{figure}[t]
    \centering
    \includegraphics[width=1\linewidth]{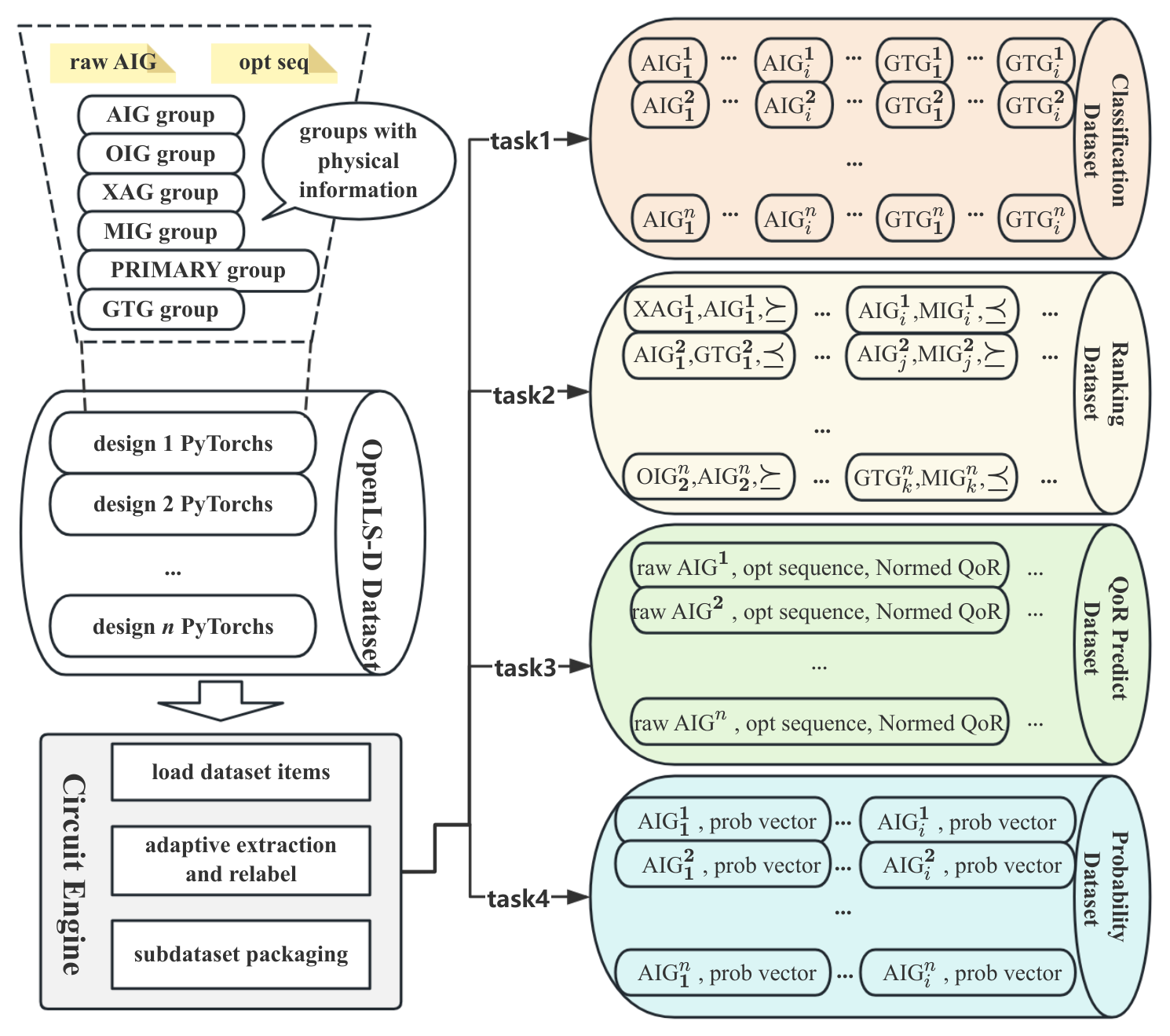}
\vspace{-0.5cm}
    \caption{Adaptive sub-dataset extraction framework of OpenLS-D-v1.}
    \label{fig:dataset}
\vspace{-0.5cm}
\end{figure}

In this section, we outline the formulation of the downstream tasks and present the experimental results for each.
\cref{fig:tasks} illustrates the context of the four selected downstream tasks—circuit classification, circuit ranking, QoR prediction, and probability prediction—and their potential applications within the logic synthesis flow.
Each of these tasks utilizes a unified adaptive sub-dataset extraction framework derived from OpenLS-D-v1.

\subsection{Adaptive Dataset Extraction Framework}

\cref{fig:dataset} illustrates the framework for adaptive dataset extraction from OpenLS-D-v1, tailored for various downstream tasks.
Leveraging the capabilities of the previously mentioned circuit engine, we can simulate the behavior of the original Boolean circuits using the specially defined ``Circuit'' class.
For a specific task within logic synthesis, the circuit engine initially loads the relevant portions of the dataset as input data for further processing.
Subsequently, an adaptive function, ``\textit{load\_adaptive\_subdataset(db:OpenLS-D-v1)}'', systematically traverses each recipe of the Boolean circuits within each item to calculate and label the target items required for the task.
These labeled items are then packaged into a sub-dataset dedicated to that specific task.

For example,  if users wish to access only the ``ABC-AIG'' related data specifically, then merely the ``\textit{raw.pt}'' and ``\textit{abc.aig.pt}'' files are required for the process.
Furthermore, \cref{fig:dataset} illustrates the items of the four sub-datasets tailored to the selected tasks.
It is crucial to implement necessary operations within the ``Circuit'' class to cater to specific task requirements.
The extraction of each sub-dataset depicted in \cref{fig:dataset} will be discussed in detail in the following subsections.

\subsection{Environment Setup}
The experimental environment for the following tasks is as follows:
The hardware configuration: CPU~(Intel Xeon Platinum 8380 CPU with 160 cores), Memory~(512 GB RAM), GPU~(NVIDIA A100 with 40 GB VRAM), while the software configuration: Operation System~(Ubuntu 20.04.6), PyTorch~(2.0.1), CUDA~(12.0), torch\_geometry~(2.3.1), scikit-learn~(1.2.2), pandas~(1.5.3), and matplotlib~(3.7.1).
This high-performance setup provides a robust platform for conducting and evaluating experiments efficiently, ensuring the smooth handling of large datasets and complex computations. However, it is worth noting that the downstream tasks in this study are not resource-intensive and utilize only a fraction of the system’s capacity for training.

\subsection{Task1: Circuit Classification}
\label{sec:tasks:task1}

\subsubsection{Problem Formulation}
The circuit classification task represents the basic attribute of the circuit analysis tasks, and it can be formulated by the following:
\begin{italics}
Given a set of Boolean circuits $\{\mathcal{C}_i\}_{i=1}^n$, classify these circuits into $k~(k\le n)$ classes, and each class follows the property of the defined Boolean equivalence, thus, $\{\mathcal{C}_i\} \equiv \{\mathcal{C}_j\}$ in the same class, otherwise, $\{\mathcal{C}_i\} \neq \{\mathcal{C}_j\}$.
\end{italics}

% In this task, the Boolean equivalence criterion is the only consideration for checking whether two circuits are in the same class.
% On the other hand, the classification of Boolean circuits is one of the fundamental tasks in the topic of Boolean circuit analysis.

\subsubsection{Dataset Adaptation}
Task 1 part of \cref{fig:dataset} shows the components of the circuit classification dataset.
The different boolean representations of one design perform the same Boolean function, as shown in \cref{fig:component}, thus, all these variations need to be in the same classes with the same label.
In this task, the selected designs are labeled with continuous natural numbers from 0 to $n$, where $n$ is the number of classes.

\subsubsection{Solution: Circuit Graph Embedding Learning and Classification}
\begin{figure}[t]
    \centering
    \includegraphics[width=1\linewidth]{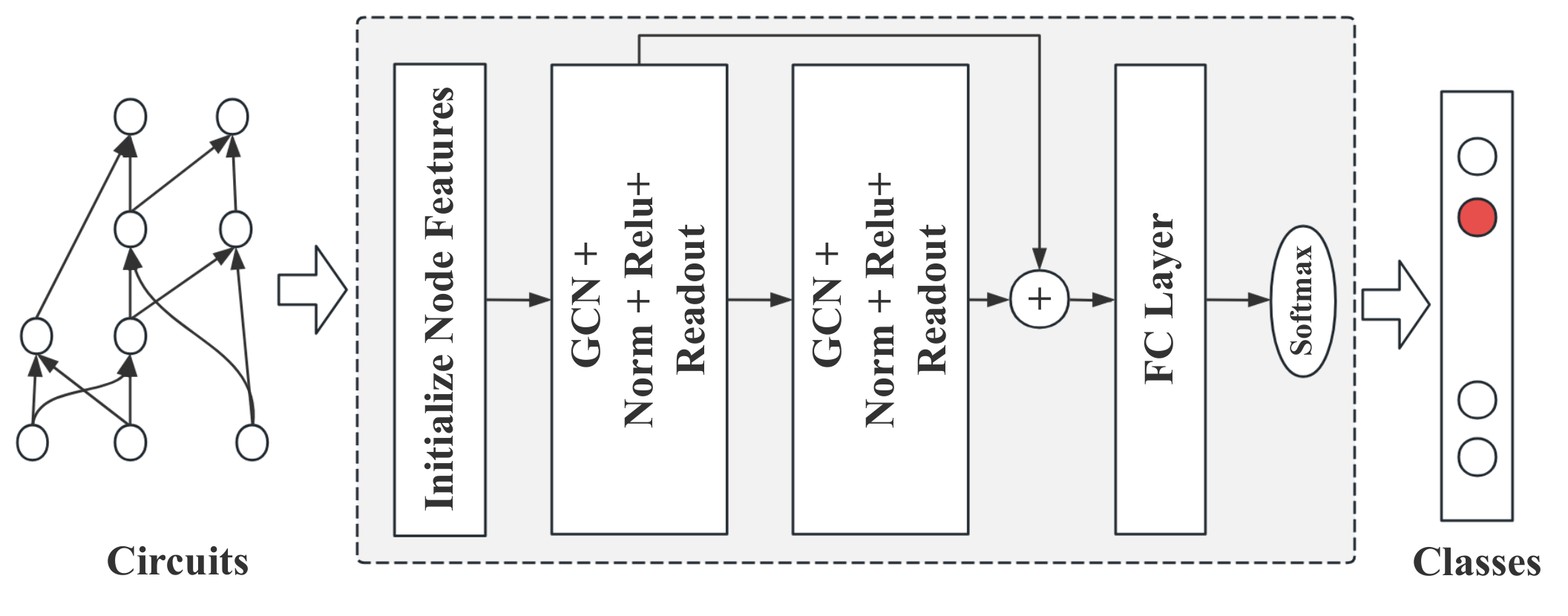}
\vspace{-0.5cm}
    \caption{The GNN-based Model for the circuit classification task.}
    \label{fig:task:solution:class}
\vspace{-0.5cm}
\end{figure}
As shown in \cref{fig:task:solution:class}, a typical two-layer GCN-based graph embedding solution is given for this task.
It comprises two steps: the preprocessing step, and the learning step.
The first preprocessing initializes the node embedding by the node type and its truth table to embed sufficient circuit information.
The following learning step aggregates the node's embedding to the graph embedding, then an MLP layer is used to predict the class number of current graph embedding.
The cross-entropy function is used as the loss function here.

\subsubsection{Experimental Results}

\begin{figure}[t]
\vspace{-0.5cm}
\centering
\subfigure[Training circuits (800/class)]{
\includegraphics[width=0.45\linewidth]{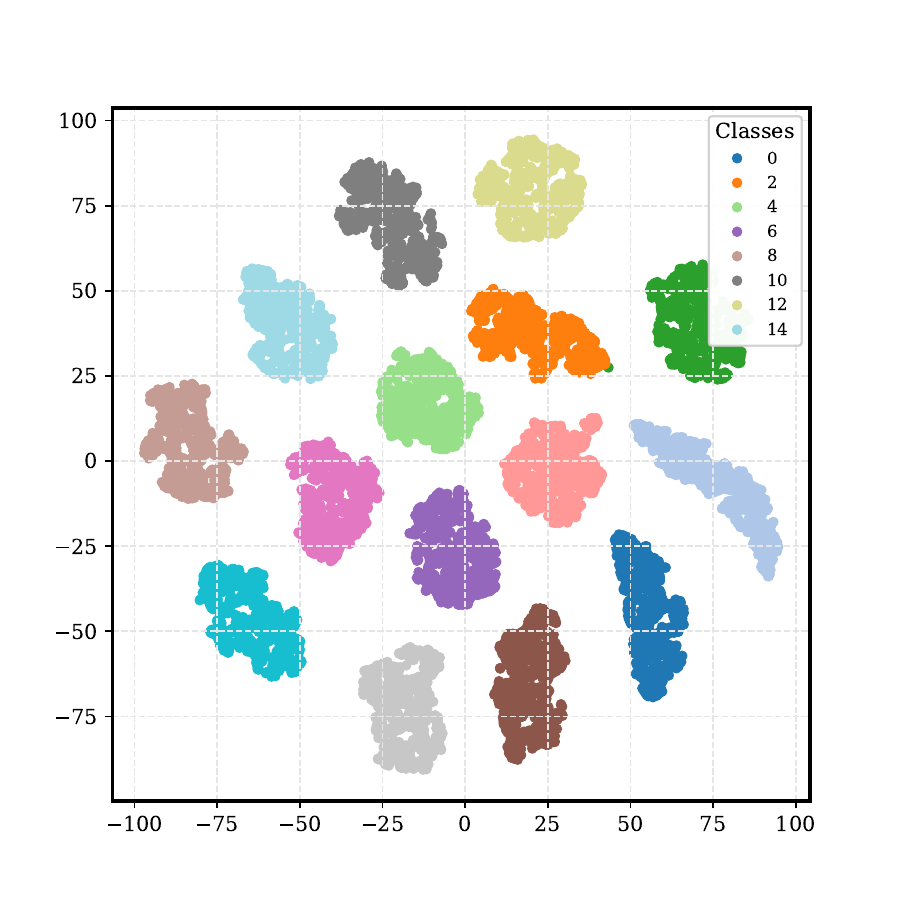}
}
\hfill
\subfigure[Test circuits (200/class)]{
\includegraphics[width=0.45\linewidth]{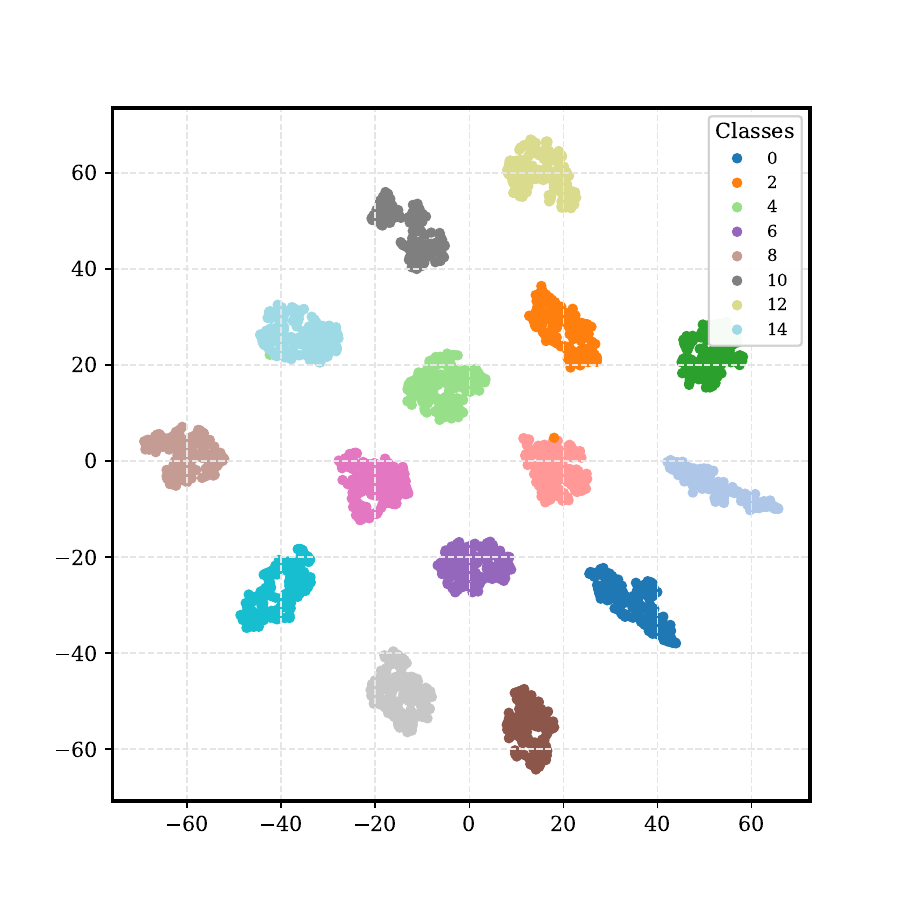}
}
\caption{The t-SNE visualization for the circuit classification task with 15 designs. Labels to designs: 0~(\textit{router}), 1~(\textit{usb\_phy}), 2~(\textit{cavlc}), 3~(\textit{adder}), 4~(\textit{systemcdes}), 5~(\textit{max}), 6~(\textit{spi}), 7~(\textit{wb\_dma}), 8~(\textit{des3\_area}), 9~(\textit{tv80}), 10~(\textit{arbiter}), 11~(\textit{mem\_ctrl}), 12~(\textit{square}), 13~(\textit{aes}), 14~(\textit{fpu}).}
\label{fig:task:exper:class}
\vspace{-0.5cm}
\end{figure}

The experiment presented here is based on a selection of 15 designs and the reassigned labels from the OpenLS-D-v1 dataset as shown in the caption of \cref{fig:task:exper:class}.
The 80\% of each design's selected Boolean circuits are used for training, and the remaining are used for validation.

The hyperparameters used in this experiment are as follows: input feature size of 64, hidden feature size of 128, learning rate of 0.0001, learning decay rate of 1e-5, and batch size of 16.
By the 10th epoch, the test accuracy reaches approximately 99.8\%.
\cref{fig:task:exper:class} presents the t-SNE visualization of the circuit classification results, clearly demonstrating distinct and independent distributions for each class.
This task shows strong potential for analyzing circuit characteristics effectively.

\subsection{Task2: Circuit Ranking}
\label{sec:tasks:task2}

\subsubsection{Problem Formulation}
Different Boolean representations of one certain design can lead to different QoRs for their corresponding gate-level netlist.
We can formulate this ranking problem by the following:
\begin{italics}
We can say that the Boolean circuits $\mathcal{C}_{0} \preceq \mathcal{C}_{1}$ only if the OoR of $\mathcal{C}_{0}$ is better than $\mathcal{C}_{1}$, otherwise, $\mathcal{C}_{0} \succeq \mathcal{C}_{1}$.
The QoR can be defined by the timing, area, power, or other criteria of the following EDA steps. 
\end{italics}

However, the technology mapping, timing analysis, and other physical design steps are time-consuming.
If we can find the best presentation of the current design, it can improve the efficiency of EDA.
In this task, we only focus on the circuit ranking problem before the technology mapping step.

\subsubsection{Dataset Adaptation}
Task 2 part of \cref{fig:dataset} shows the components of the circuit ranking dataset.
The partial order is defined as the following:
\begin{equation}
\vspace{-0.2cm}
\scriptsize
\nonumber
\begin{aligned}
Timing^{\mathcal{C}_0} < Timing^{\mathcal{C}_1}& \Rightarrow \mathcal{C}_{0} \preceq \mathcal{C}_{1}; \\
Timing^{\mathcal{C}_0} > Timing^{\mathcal{C}_1}& \Rightarrow \mathcal{C}_{0} \succeq \mathcal{C}_{1}; \\
Timing^{\mathcal{C}_0} = Timing^{\mathcal{C}_1}&, then: \\
Area^{\mathcal{C}_0} &< Area^{\mathcal{C}_1} \Rightarrow \mathcal{C}_{0} \preceq \mathcal{C}_{1}; \\
Area^{\mathcal{C}_0} &> Area^{\mathcal{C}_1} \Rightarrow \mathcal{C}_{0} \succeq \mathcal{C}_{1}; \\
Area^{\mathcal{C}_0} &= Area^{\mathcal{C}_1}, pass \\
\end{aligned}
\end{equation}
Then, we construct the graph pairs and their partial order by the combination of the logic types as listed at \cref{tab:fcs}, and we only consider the $\mathcal{C}_{0} \preceq \mathcal{C}_{1}$ status as the $\mathcal{C}_{0} \succeq \mathcal{C}_{1}$ can be converted to $\mathcal{C}_{1} \preceq \mathcal{C}_{0}$.
The tie condition of timing and area is also not considered.

\subsubsection{Solution: Pair-wise Graph Ranking}

\begin{figure}[t]
    \centering
    \includegraphics[width=1\linewidth]{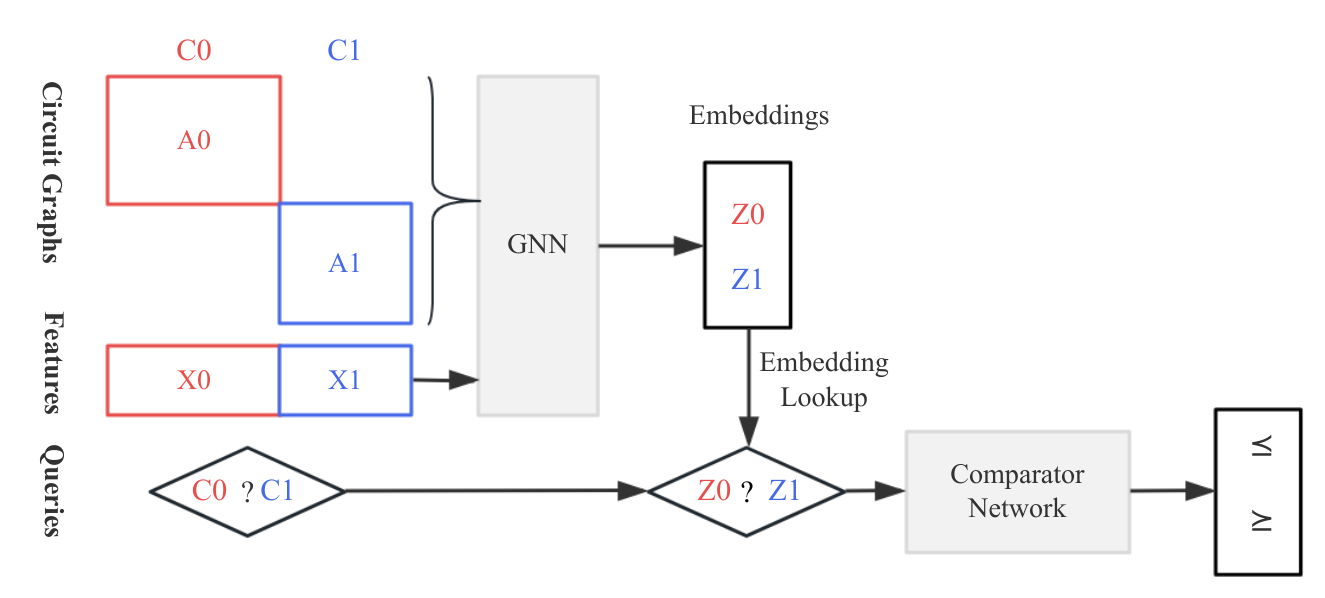}
\vspace{-0.5cm}
    \caption{The pair-wise graph ranking model.}
    \label{fig:task:solution:ranking}
% \vspace{-0.5cm}
\end{figure}

\cref{fig:task:solution:ranking} illustrates the architecture of the pair-wise graph ranking model.
We first combine these two Boolean circuits $\mathcal{C}_{0}$ and $\mathcal{C}_{1}$ into a block matrix.
Then a GNN-based graph embedding will learn the combined embedding of these two circuits.
Finally, the MLP-based compassion network will tell that $\mathcal{C}_{0} \preceq \mathcal{C}_{1}$ or $\mathcal{C}_{0} \succeq \mathcal{C}_{1}$;
The binary cross entropy function is used as the loss function here.

\subsubsection{Experimental Results}

\begin{figure}[h]
\vspace{-0.5cm}
\centering
\subfigure[Test, Epoch=1, Acc=92.51\%]{
\includegraphics[width=0.45\linewidth]{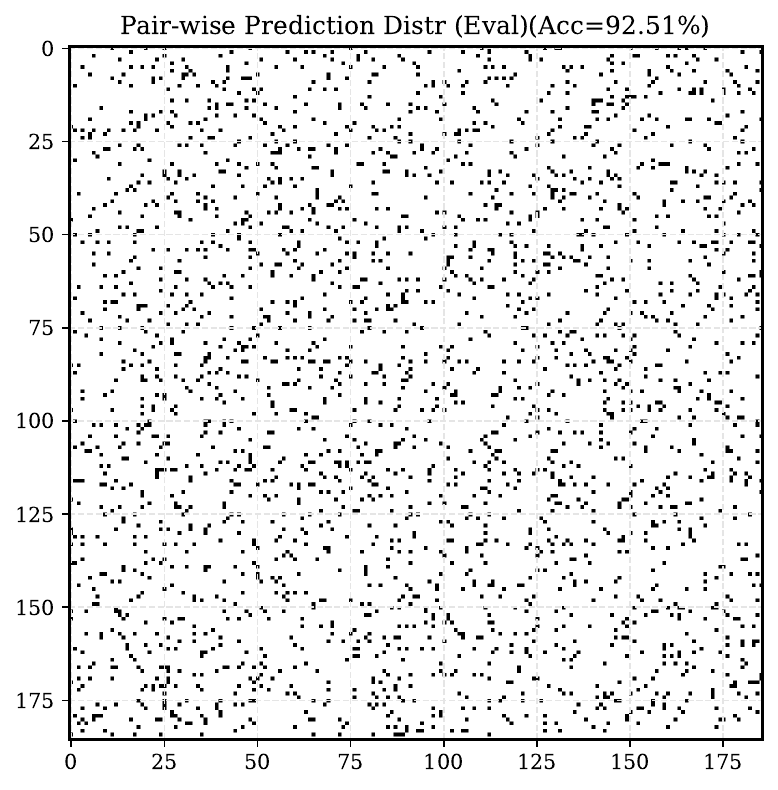}
}
\hfill
\subfigure[Test, Epoch=20, Acc=99.49\%]{
\includegraphics[width=0.45\linewidth]{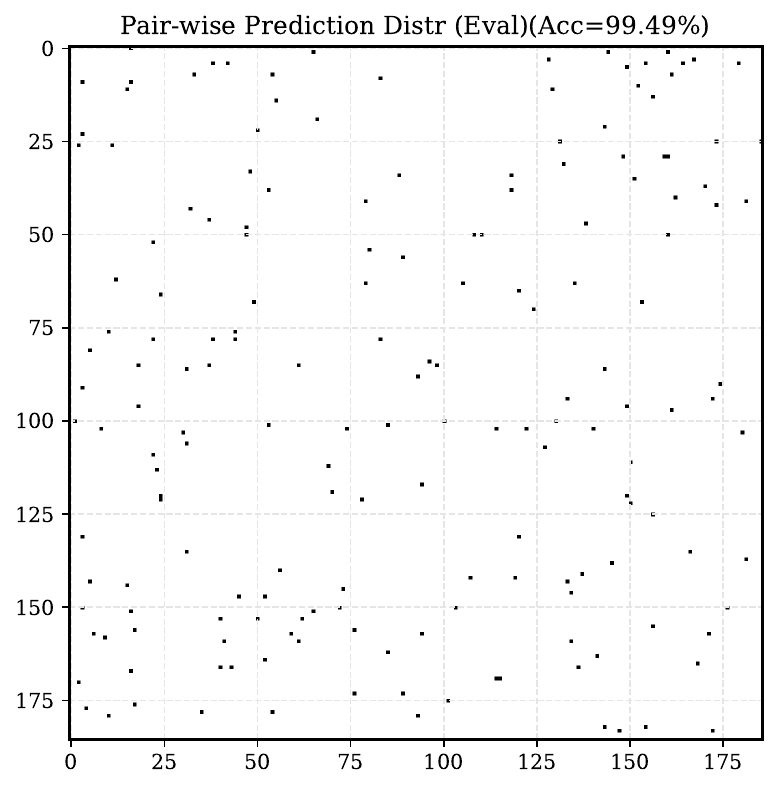}
}
\caption{Pair-wise prediction distribution of the evaluation.}
\label{fig:task:exper:ranking}
\vspace{-0.2cm}
\end{figure}

\begin{table}[h]
\centering
\caption{Performance comparison with different Graph Embedding Model~(epoch = 20).}
\begin{tabular}{c|c|c|c}
\toprule
\diagbox{Metric}{Model} &   GCNConv   &  GraphSAGE  &  GINConv \\
\midrule
\textit{BCE\_loss} &  5.09    &  \textbf{4.31}    &  4.43    \\
\textit{Accuracy}  &  99.43\% &  \textbf{99.49\%} &  99.47\% \\
\textit{Precision} &  0.9939  &  \textbf{0.9949}  &  0.9944  \\
\textit{Recall}    &  0.4993  &  \textbf{0.4995}  &  0.4969  \\
\textit{F1-score}  &  0.6647  &  \textbf{0.6651}  &  0.6627  \\
\bottomrule
\end{tabular}
\label{tab:task:exper:ranking}
% \vspace{-0.5cm}
\end{table}

The experiment presented here is based on a selection of 10 designs from the OpenLS-D-v1 dataset: \textit{ctrl}, \textit{router}, \textit{int2float}, \textit{ss\_pcm}, \textit{usb\_phy}, \textit{sasc}, \textit{cavlc}, \textit{simple\_spi}, \textit{priority}, and \textit{i2c}. 
Using the previously defined partial order, approximately 120,000 pairs were extracted from OpenLS-D-v1 for analysis. Of each design’s selected Boolean circuits, 70\% were used for training, and the remaining for validation.

The hyperparameters used in this experiment are as follows: input feature size of 64, hidden feature size of 128, learning rate of 0.0001, learning decay rate of 1e-5, and batch size of 32.
\cref{fig:task:exper:ranking} illustrates the distribution of pair-wise prediction outcomes on the test dataset, 
with black nodes indicating incorrect predictions and white nodes indicating correct ones. 
The results indicate that all three models achieve high accuracy and effective pair-wise ranking predictions, demonstrating their validity for this application.
% The results show that by epoch 20, the accuracy reaches around 99.49\% by the GraphSAGE model.
% Furthermore, the \cref{tab:task:exper:ranking} compares the performance of three different graph embedding models (GCNConv, GraphSAGE, and GINConv) in the pair-wise circuit ranking task.

\subsection{Task3: QoR Prediction}
\label{sec:tasks:task3}

\subsubsection{Problem Formulation}
% In the logic synthesis phase, it is necessary to transform the structure within the synthesis tool into an optimized sequence applied to the circuit to reduce its area and delay. 
% However, execution of the logic synthesis optimization sequence is extremely time-consuming. 
% Therefore, predicting the quality of results (QoR) based on the comprehensive optimization sequence of the circuit can help engineers to identify superior optimization sequences in a faster way. 
% This task leverages deep learning methods to predict the QoR of unknown pair of circuit optimization sequences.

\cref{fig:qor_incr_logic} and its corresponding observation 3 illustrate the motivation behind the QoR prediction task: once an adequate QoR distribution for a circuit is obtained, it is possible to make predictions about the inputs.
\begin{italics}
Given a known QoR distribution $\mathcal{D}$, a Boolean circuit $\mathcal{C}$, and an optimization sequence $S$, the objective is to predict the QoR of $\mathcal{C}$ with $S$.
\end{italics}

\subsubsection{Dataset Adaptation}
Task 3 part of \cref{fig:dataset} presents the profile of the sub-dataset used for QoR prediction.
In the proposed framework, an ASIC gate-level netlist is generated for each optimized Boolean network.
The Area and Timing are used as the QoR for the unoptimized Boolean networks and their respective optimization sequences.
Consequently, each data item is organized as \{unoptimized Boolean network, optimization sequence, Area, Timing\}.
Notably, due to the specific requirements of this task, different designs share the same optimization sequence within the same recipe index.

\subsubsection{Solution: QoR Distribution Convergence Learning}

\begin{figure}[t]
    \centering
    \includegraphics[width=1\linewidth]{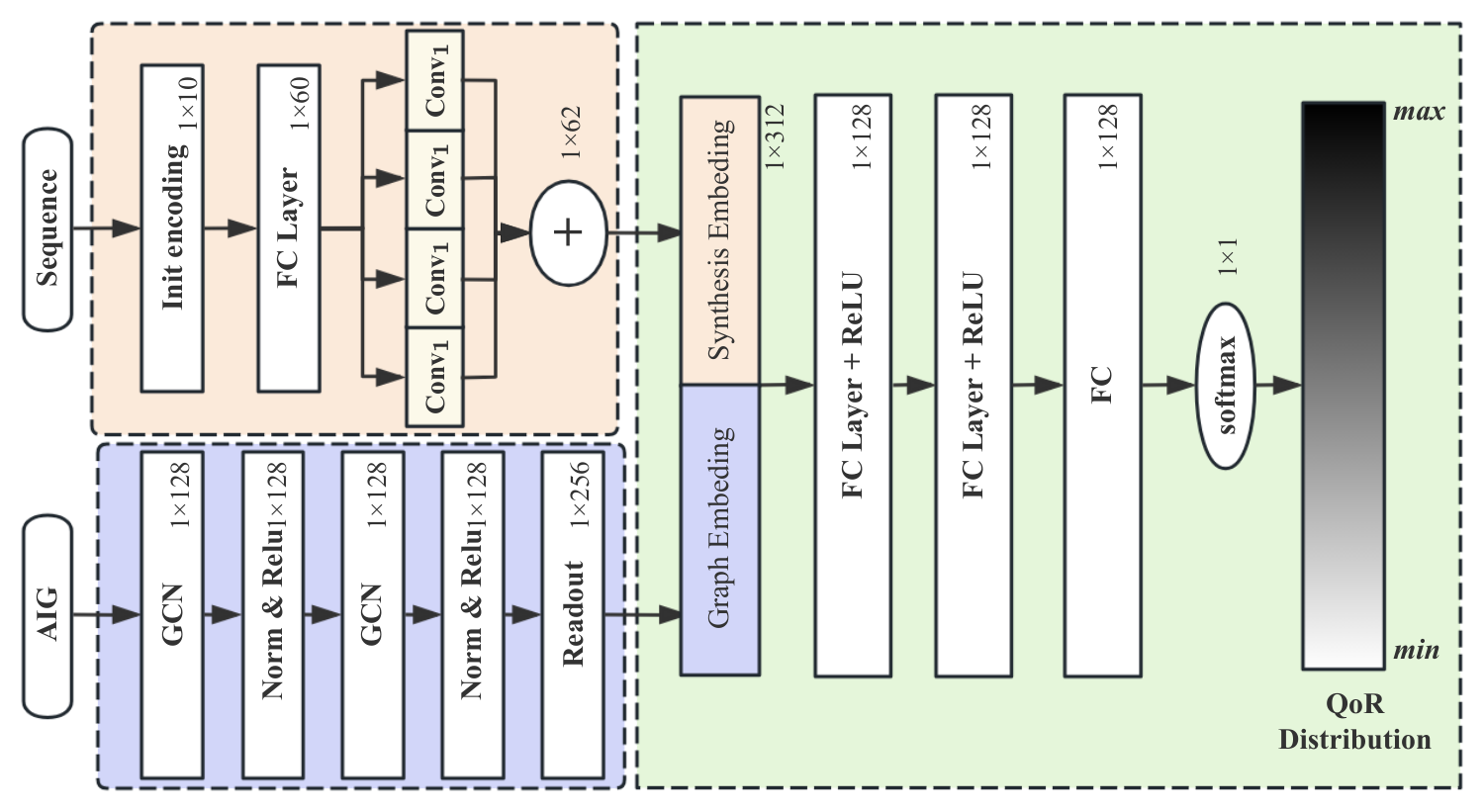}
\vspace{-0.5cm}
    \caption{Convolutional networks for QoR synthesis formulation prediction.}
    \label{fig:task:solution:qor}
\vspace{-0.5cm}
\end{figure}

% In this task, we propose a Graph Neural Network (GNN)-based model for circuit diagram embedding learning and recipe embedding. The model consists of three main steps: graph encoding, sequence encoding, and prediction after combining graph and sequence. We utilize Graph Neural Networks (GNNs) to extract structural features from circuit diagrams. The optimization sequences, represented as numerically encoded synthesis recipes, are processed through a linear layer and then a set of one-dimensional convolutional filters to extract features. The connected graph-level and optimized sequence embeddings are subsequently passed to a fully connected layer to perform regression analysis. The detailed neural network is shown in the figure \cref{fig:subdataset:QoR predict} .

\cref{fig:task:solution:qor} illustrates the proposed QoR prediction architecture.
Each feature dimension is labeled in each layer.
First, the input AIG and the optimization sequence are each embedded separately.
The AIG is embedded using a standard GNN-based graph embedding approach, incorporating global mean and sum pooling as the readout layer.
The optimization sequence, represented by numerically encoded synthesis recipes, is processed through a linear layer followed by four convolutional filters with dimensions \{1$\times$14, 1$\times$15, 1$\times$16, 1$\times$17\}, designed to extract relevant features.
These two embeddings are then concatenated and fed into an MLP-based distribution learning module.
Finally, a softmax activation function in the output layer predicts the position within the overall distribution, providing the QoR prediction.

\subsubsection{Experimental Results}

\begin{table}[t]
\centering
\scriptsize
\setlength{\tabcolsep}{1.2pt}
\caption{MAPE Results for QoR Prediction (\%)}
\begin{tabular}{cc|cc|cc}
\toprule
\multicolumn{2}{c|}{\makecell{Variant1 \\Seen Design, Unseen Recipe}} & \multicolumn{2}{c|}{\makecell{Variant2 \\Unseen Design, Seen Recipes}} &  \multicolumn{2}{|c}{\makecell{Variant3 \\Unseen IC and Recipes}}\\
\midrule
Area & Timing & Area & Timing & Area & Timing \\
\midrule
0.69 & 7.87 & 1.06 & 6.50 & 1.17 & 6.49 \\
\bottomrule
\end{tabular}
\label{tab:task:exper:qor}
\vspace{-0.2cm}
\end{table}

\begin{figure}[t]
\vspace{-0.5cm}
\centering
\subfigure[\textit{cavlc(\textbf{area})}]{
\includegraphics[width=0.21\linewidth]
{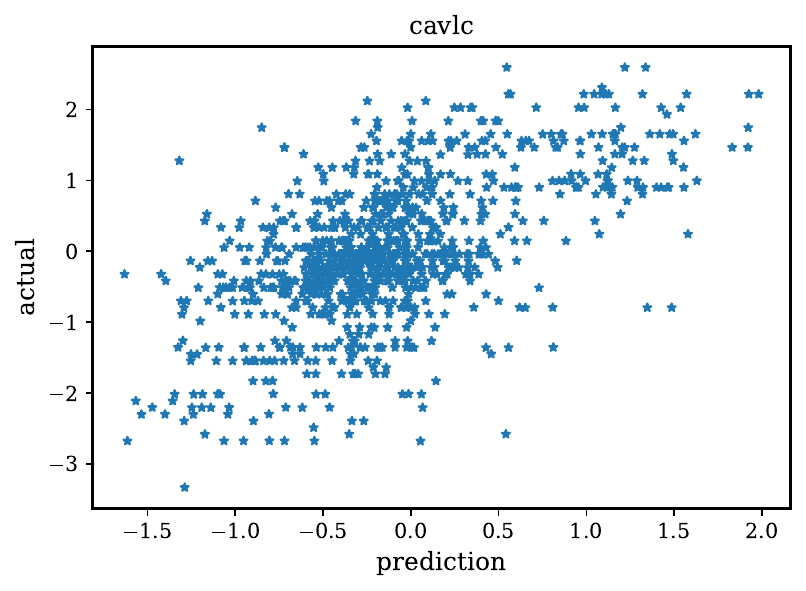}
}
\hfill
\subfigure[\textit{max(**)}]{
\includegraphics[width=0.21\linewidth]
{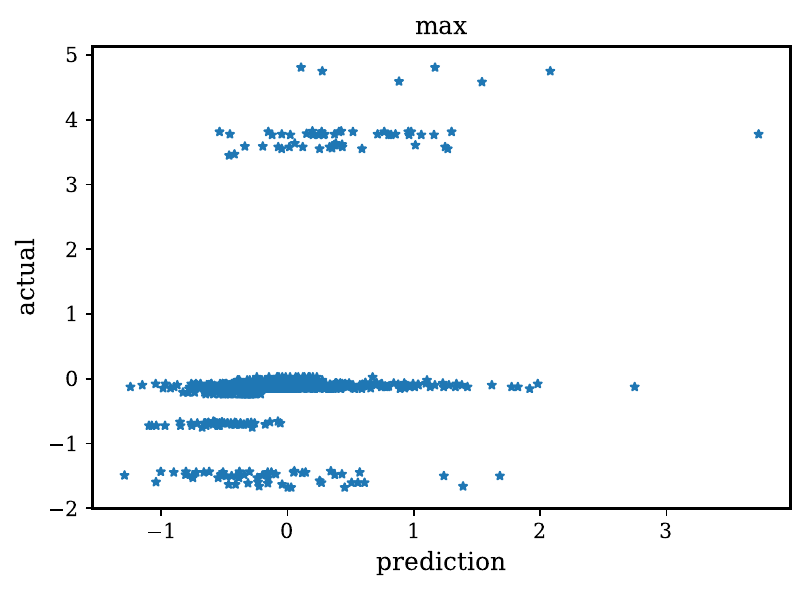}
}
\hfill
\subfigure[\textit{router(**)}]{
\includegraphics[width=0.21\linewidth]
{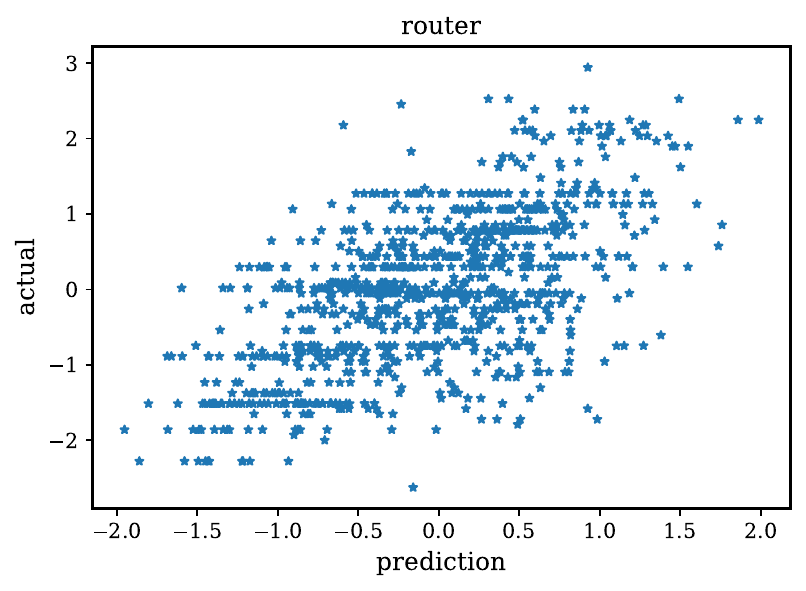}
}
\hfill
\subfigure[\textit{fpu(**)}]{
\includegraphics[width=0.21\linewidth]
{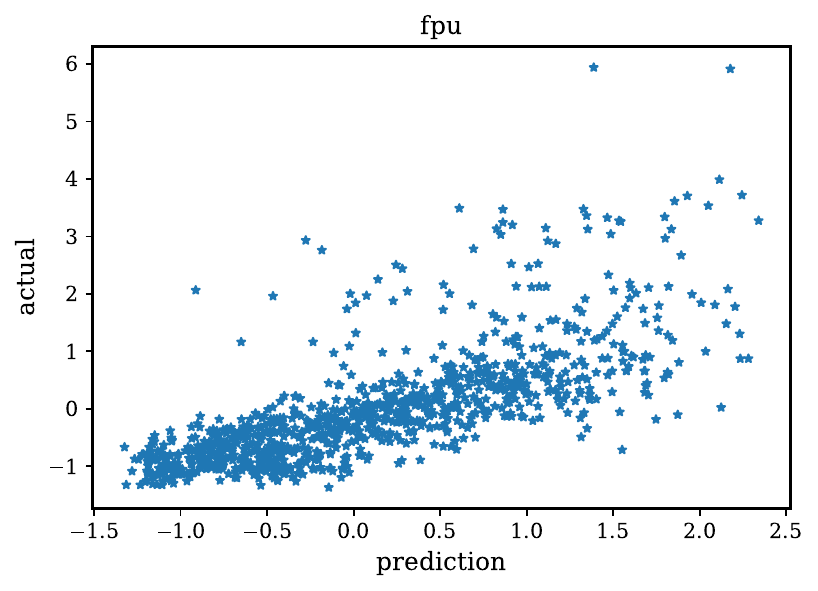}
}
\\
\subfigure[\textit{cavlc(\textbf{timing})}]{
\includegraphics[width=0.21\linewidth]
{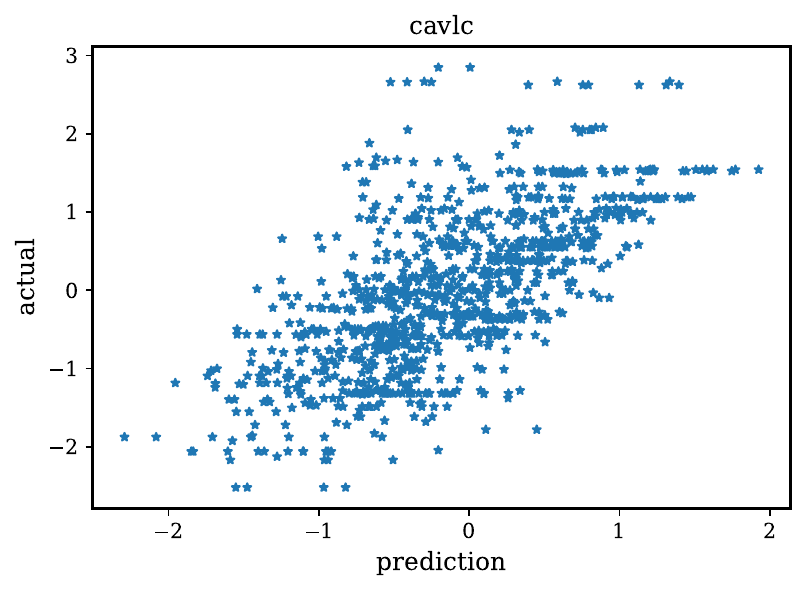}
}
\hfill
\subfigure[\textit{max(**)}]{
\includegraphics[width=0.21\linewidth]
{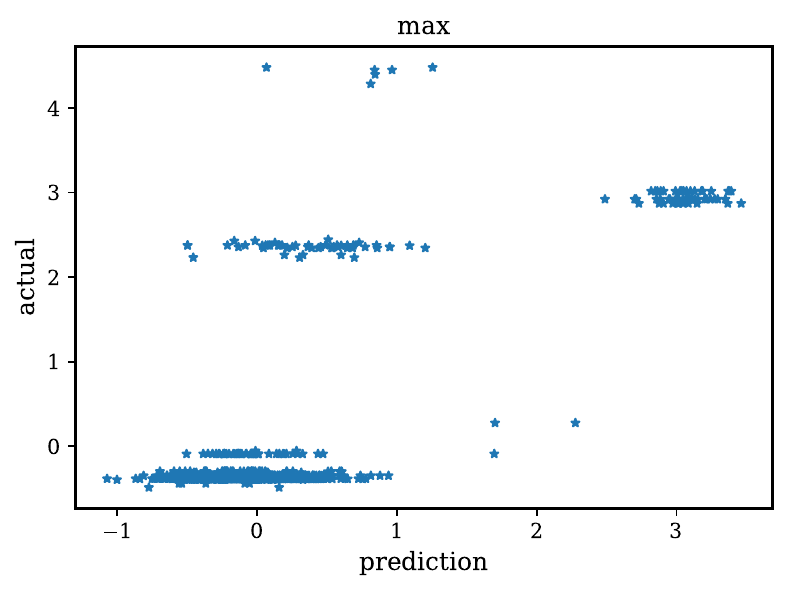}
}
\hfill
\subfigure[\textit{router(**)}]{
\includegraphics[width=0.21\linewidth]
{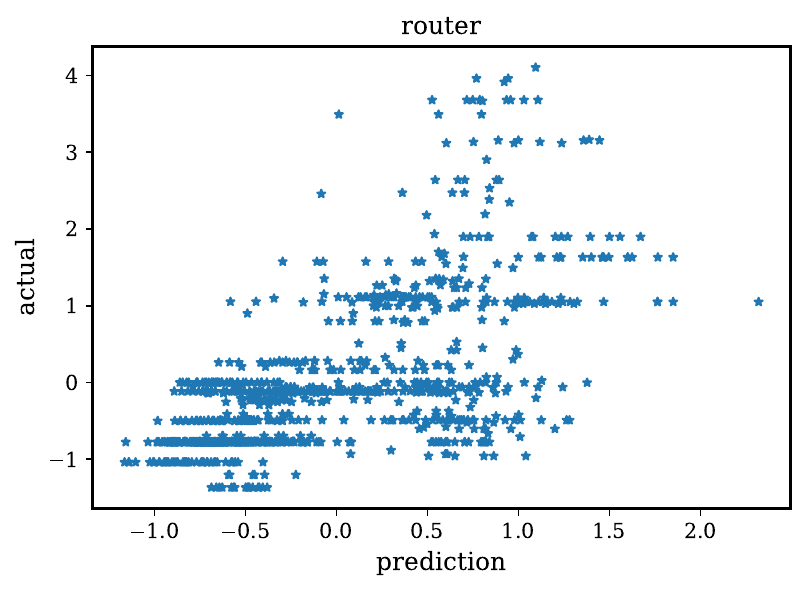}
}
\hfill
\subfigure[\textit{fpu(**)}]{
\includegraphics[width=0.21\linewidth]
{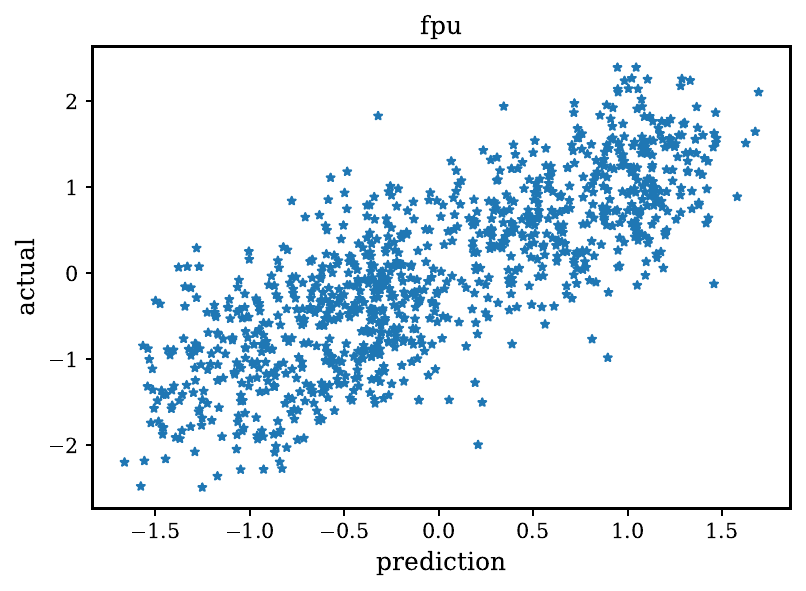}
}
\caption{QoR predictions vs. ground truth of variant 1. (a)-(d) focus on the prediction of area, while (e)-(h) focus on timing.}
\label{fig:task:exper:qor}
\vspace{-0.5cm}
\end{figure}

There are three variants of the QoR prediction tasks, and the detailed configuration of each variant is outlined as follows:
\begin{itemize}
    \item \textit{\textbf{Variant1}}: Predicting QoR of synthesizing unseen Recipes. All designs and the AIG outputs of 700 synthesis recipes are used as the training set, and the remaining 300 recipes of each design are used as the test set;
    \item \textit{\textbf{Variant2}}: Predicting QoR of synthesizing unseen Designs. We select 20 small designs as the training set, and the remaining 14 larger designs as the test set;
    \item \textit{\textbf{Variant3}}: Predicting QoR on unseen Design-Recipe Combination. We randomly selected 70\% of the synthesis recipes across all designs.
\end{itemize}

\cref{tab:task:exper:qor} shows the mean absolute percentage error~(MAPE) for the above three variant tasks.
It indicates that the model achieved high predictive accuracy in area prediction, while timing prediction posed a challenge, with slightly lower accuracy compared to area prediction. 
This discrepancy may stem from the complexity of timing prediction, necessitating more refined feature engineering and model tuning. 
Overall, the model demonstrated good generalization performance with unknown circuit and optimization sequence pairs, indicating a certain level of robustness in the model.
\cref{fig:task:exper:qor} visualizes the relationship between the prediction and ground truth.

% As shown in  \cref{fig:task:qor prediction vs. ground truth}, the model's performance varied across different circuits.
% For certain circuits, such as fpu, the model's inferred results closely matched the actual QoR values, showing high predictive accuracy. However, for other circuits like max, the model struggled to distinguish the QoR values of different optimization sequences, possibly indicating a significant difference between the training data distribution and the test data distribution, which limits the model's generalization capabilities. This finding underscores the importance of collecting comprehensive circuit data with diverse feature distributions during the model training process.
\vspace{-0.2cm}
\subsection{Task4: Probabilistic Prediction}
\label{sec:tasks:task4}

\subsubsection{Problem Formulation}
The Probabilistic Prediction task is a gate-level task that predicts the truth-table probability of the gate in the circuit~\cite{DeepGate}.
It can predict the logical probability of the gate without computing the truth table and can be formulated by:
\begin{italics}
Given a Boolean Circuit $\mathcal{C}$, the logic probability of a gate $v$ is defined as the frequency of the $1$s in the gate's truth table.
\end{italics}

\subsubsection{Dataset Adaptation}
Task 4 part of \cref{fig:dataset} shows the dataset components of probability prediction.
For each Boolean network in the OpenLS-D-v1 items, we label each Boolean network with a probability vector~(the probability of each node) which is computed by random simulation of the Boolean network by ``\textit{simulate(circuit:Circuit)}'' with enough activate vectors of PIs.

\subsubsection{Solution: Node Embedding Learning}

\begin{figure}[t]
    \centering
    \includegraphics[width=1\linewidth]{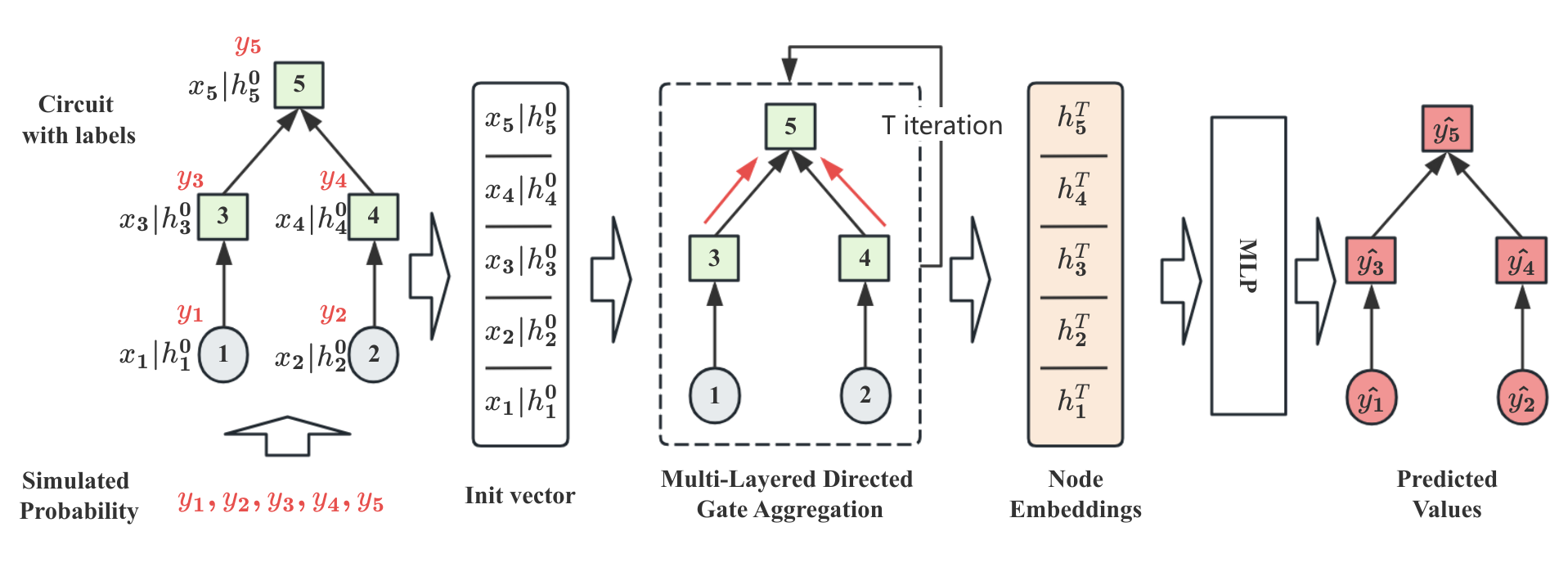}
\vspace{-0.5cm}
    \caption{The node embedding model.}
    \label{fig:task:solution:probability}
\vspace{-0.2cm}
\end{figure}

Instead of computational simulation of the circuit, the probability can be easily attained by the Node embedding-based method.
\cref{fig:task:solution:probability} shows the node embedding learning-based probability prediction method.
For a given Boolean network with the probability vector~{$x_1, x_2, ..., x_n$}, the first step is to initialize the node feature vector~{$h^0_1, h^0_2, ..., h^0_n$} for the nodes;
Then a $T$-layered directed gate aggregation is performed to learn the feature of each node.
Finally, an MLP layer is to reduce each node's feature to 1 to predict the logic probability of each node.
The loss function is the average prediction error~(PE) loss, followed by DeepGate. 

\subsubsection{Experimental Results}

\begin{table}[h]
\centering
\scriptsize
\setlength{\tabcolsep}{4.0pt}
\caption{Average prediction error and time comparison by the different node embedding methods, while the comparison column represents DeepGate2/GraphSAGE.}
\begin{tabular}{c|ll|ll|ll}
\toprule
\multirow{2}{*}{\diagbox{recipes}{Model}} & \multicolumn{2}{c|}{GraphSAGE} & \multicolumn{2}{c|}{DeepGate2} & \multicolumn{2}{c}{Comparison} \\
                                          &     PE & Time(s)               &     PE & Time(s)            & PE & Time(x)  \\
\midrule
\textit{100}                              &  0.011  & 0.050                &    0.0082  & 1.42 & 24\%{\color{darkgreen} $\uparrow$} & 27.40{\color{red} $\downarrow$}  \\          
\textit{500}                              &  0.002  & 0.098                &    0.001   & 1.48 & 50\%{\color{darkgreen} $\uparrow$} & 14.10{\color{red} $\downarrow$}  \\          
\textit{1000}                             &  0.0008 & 0.038                &    0.0002  & 2.20 & 75\%{\color{darkgreen} $\uparrow$} & 56.89{\color{red} $\downarrow$} \\          
\bottomrule
\end{tabular}
\label{tab:task:exper:pp}
\vspace{-0.2cm}
\end{table}

The experiment presented here is based on a selection of 10 designs from the OpenLS-D-v1 dataset: \textit{ctrl}, \textit{router}, \textit{int2float}, \textit{ss\_pcm}, \textit{usb\_phy}, \textit{sasc}, \textit{cavlc}, \textit{simple\_spi}, \textit{priority}, and \textit{steppermotordrive}. 
The recipe size of the dataset ranges from \textit{100} to \textit{1000}, showing the scalability and adaptability of the dataset. 
A 70-30 split is used for training and validation, respectively.

The hyperparameters used in this experiment are as follows: input feature size of 64, hidden feature size of 128, learning rate of 0.001, learning decay rate of 1e-4, and batch size of 64 for fast training.
The comparison between different methods can be seen in \cref{tab:task:exper:pp}, where DeepGate outperforms the GraphSAGE model in various scales of the dataset.

% \begin{figure}[h]
% \vspace{-0.5cm}
% \centering
% \subfigure[]{
% \includegraphics[width=0.45\linewidth]
% {figures/task-pp/pp_loss_train.pdf}
% }
% \hfill
% \subfigure[]{
% \includegraphics[width=0.45\linewidth]
% {figures/task-pp/pp_loss_test.pdf}
% }
% \caption{Training (a) and testing (b) loss of DeepGate model for Probabilistic Prediction tasks on the OpenLS-D-v1 dataset.}
% \label{fig:task:exper:pp}
% \vspace{-0.3cm}
% \end{figure}

%%%%%%%%%%%%%%%%%%%%%%%%%%%%%%%%%%%%%%%%%%%%%%%%%%%%%%%%%%%%%%%%%%%%%%%%%%%%
% discussion
%%%%%%%%%%%%%%%%%%%%%%%%%%%%%%%%%%%%%%%%%%%%%%%%%%%%%%%%%%%%%%%%%%%%%%%%%%

\section{Discussion}
\label{sec:discussion}

\begin{table}[t]
\centering
% \scriptsize
\setlength{\tabcolsep}{1.2pt} % 设置列间距为3pt
\caption{The Tasks comparison between the logic synthesis-related Datasets.}
\begin{tabular}{c|c|c|c|c}
\toprule
\diagbox{Tasks}{Dataset}  & OpenABC-D~\cite{openabc-d} & Deepgate~\cite{DeepGate} & Gamora~\cite{Gamora} & \textbf{OpenLS-D-v1} \\
\midrule
\makecell{Circuit \\Classification}  & $\surd$  & $\times$ & $\times$ & $\surd$  \\ \hline
\makecell{Node \\Classification}     & $\times$ & $\times$ & $\surd$  & $\surd$  \\  \hline
\makecell{QoR \\Prediction}          & $\surd$  & $\times$ & $\times$ & $\surd$  \\  \hline
\makecell{Circuit \\Ranking}         & $\times$ & $\times$ & $\times$ & $\surd$  \\  \hline
\makecell{probability \\Prediction}  & $\times$ & $\surd$  & $\times$ & $\surd$  \\
\bottomrule
\end{tabular}
\label{tab:discussion:comparison}
\vspace{-0.5cm}
\end{table}

The comparison table in \cref{tab:discussion:comparison} illustrates the comprehensive task support offered by the OpenLS-D-v1 dataset in contrast to other datasets such as OpenABC-D, DeepGate, and Gamora. 
Each dataset was developed with specific objectives, and they serve different roles in the enhancement of logic synthesis processes through machine learning.

OpenLS-D-v1 distinguishes itself by offering comprehensive task coverage, essential for developing versatile and generalized machine learning models within the logic synthesis domain. 
This wide-ranging support enables diverse experimental setups, paving the way for novel approaches to circuit design and analysis within a unified dataset framework.
This broad task support enables researchers to explore new methodologies and improve existing processes.

In summary, OpenLS-DGF is capable of supporting a variety of machine learning tasks highlighting its potential as a general resource and standardized process in the field of logic synthesis. 
Additionally, the OpenLS-D-v1 dataset further enhances this by providing a versatile foundation for future research and innovation.

%%%%%%%%%%%%%%%%%%%%%%%%%%%%%%%%%%%%%%%%%%%%%%%%%%%%%%%%%%%%%%%%%%%%%%%%%%%%
% conclusion
%%%%%%%%%%%%%%%%%%%%%%%%%%%%%%%%%%%%%%%%%%%%%%%%%%%%%%%%%%%%%%%%%%%%%%%%%%%%
\section{Conclusion}
\label{sec:conclusion}
In this paper, we begin by addressing the lack of a dataset generation flow specifically targeting multiple tasks in logic synthesis.
To overcome this limitation, we propose OpenLS-DGF, an adaptive dataset generation framework tailored for machine learning tasks within logic synthesis.
We highlight that the proposed solution framework can target multiple machine-learning tasks in logic synthesis.
We also generate the OpenLS-D-v1 dataset, created using OpenLS-DGF, and demonstrate its utility by implementing and evaluating four typical tasks on OpenLS-D-v1.
The results of these tasks validate the effectiveness and versatility of our framework.

Future work will focus on enhancing the efficiency of the generation flow and benchmarking the specific machine-learning tasks for logic synthesis.
Furthermore, we aim to integrate the machine learning models into the logic synthesis flow, contributing to an improved EDA flow for enhanced circuit performance.

\section*{Acknowledgment}
The authors would like to express their gratitude to the teams behind the related open-source tools, including Yosys~\cite{Yosys}, ABC~\cite{ABC}, LSILS~\cite{lsils}, OpenRoad~\cite{openroad}, iEDA~\cite{iEDA}, and LogicFactory~\cite{LogicFactory}.

\bibliographystyle{IEEEtran}
\bibliography{ref}

% \end{IEEEbiography}

\end{CJK}
\end{document}